%% file: main.tex
\documentclass[letterpaper]{article}
\usepackage[preprint]{aaai2027}
\usepackage[hyphens]{url}
\usepackage{graphicx}
\usepackage{natbib}
\usepackage{caption}
\usepackage{amsmath,amssymb,amsthm}
\usepackage{booktabs}
\usepackage{multirow}
\usepackage{nicefrac}
\usepackage[table]{xcolor}

\newtheorem{arctheorem}{Theorem}
\newtheorem{arcproposition}[arctheorem]{Proposition}
\newtheorem{arclemma}[arctheorem]{Lemma}
\newtheorem{arccorollary}[arctheorem]{Corollary}
\theoremstyle{definition}
\newtheorem{arcdefinition}[arctheorem]{Definition}
\newtheorem{arcassumption}[arctheorem]{Assumption}
\theoremstyle{remark}
\newtheorem{arcremark}[arctheorem]{Remark}

\title{Addressable Recall Compaction for Long Context-Window Control in AI Agents}
\author{
    Thang Dang\textsuperscript{\rm 1}\equalcontrib\corresponding,
    Yuma Ichikawa\textsuperscript{\rm 2,\rm 3}\equalcontrib\corresponding,
    Sakina Fatima\textsuperscript{\rm 1},
    Koichi Shirahata\textsuperscript{\rm 2}
}
\affiliations{
    \textsuperscript{\rm 1}Fujitsu Research of America;
    \textsuperscript{\rm 2}Fujitsu Limited Japan;
    \textsuperscript{\rm 3}RIKEN Center for AIP\\
    tdang@fujitsu.com, ichikawa.yuma@fujitsu.com, sakina.fatima@fujitsu.com, k.shirahata@fujitsu.com
}

\begin{document}

\maketitle

\begin{abstract}
Long-horizon LLM agents accumulate reasoning traces, actions, and tool observations that can eventually exceed a model's fixed context window. Existing compaction methods address this limitation by discarding, summarizing, or retrieving earlier information, but they may remove task-critical details or fail to recover them reliably. We propose \textbf{ARC} (Addressable Recall Compaction), a context-management framework that separates archival storage from active-context presentation. ARC stores tool observations in an append-only, ID-addressable log and replaces older observations with compact citations when compaction is required. The agent can subsequently use these identifiers to request stored content without re-executing the corresponding tools or depending solely on similarity-based retrieval. We evaluate ARC using Qwen3-8B with a 16k context window and Qwen3-32B with a 32k context window. On the Needle-in-a-Haystack evaluation, ARC achieves an average exact-answer accuracy of 99.40\%, compared with 88.12\% for the best-performing baseline in our evaluation. ARC also reduces estimated serving time and HBM traffic under our hardware-cost model. On the LongBench-v2 Hard subset, ARC obtains an average accuracy of 29.97\%, compared with 28.25\% for the best-performing baseline. These results indicate that explicit, address-based recall can improve information retention and serving efficiency relative to the evaluated context-management baselines under the tested settings.
\end{abstract}

\input{introduction}

\input{related_works}
\input{background}
\input{proposed_method}
\input{experimental_setup}
\input{results}
\input{discussion_limittation}
\input{conclusion}

\bibliography{aaai2027}

\appendix
\begin{center}
    {\Large \textbf{Supplementary}}
\end{center}
\input{proof}

\input{ablation}
\input{cross_platform}

\end{document}

%% file: introduction.tex
\section{Introduction}
\label{sec:intro}

In recent years, the paradigm of Large Language Models (LLMs)~\cite{masterman2024landscape} has transitioned from passive text generation~\cite{li2024pre} to active, autonomous agency. By embedding an LLM within an interactive execution loop, modern AI agents~\cite{durante2024agent} can perceive environments, reason over objectives, and execute actions via external tools, APIs, or bash shells. At each discrete time step $t$, the agent reads the cumulative historical transcript of the trajectory, synthesizes an internal monologue, and emits a structured action. The underlying runtime environment executes this action and appends the resulting observation back into the transcript, preparing the agent for the subsequent turn. Because contemporary autoregressive language models remain fundamentally stateless between independent API calls or inference passes, this monotonically growing transcript serves as the agent’s singular, monolithic working memory.

However, this architecture introduces a severe engineering and algorithmic bottleneck: the context-window limit~\cite{tang2025beyond}. Every LLM is bounded by a strict maximum sequence length that it can process at once. Because the agentic trajectory at any given turn is a literal concatenation of all prior prompts, tool outputs, and environment responses, the transcript grows continuously and without bound as the agent keeps acting.

Once the cumulative transcript outgrows this maximum capacity, the system encounters an unrecoverable failure state: the LLM can no longer process the next step, effectively paralyzing the agent. For long-horizon planning~\cite{erdogan2025plan}, complex software engineering tasks, or open-ended web browsing, hitting this context wall is not a matter of if, but when rendering even modern ultra-large context windows insufficient.

\begin{figure*}[t]
\centering
\includegraphics[width=0.9\textwidth]{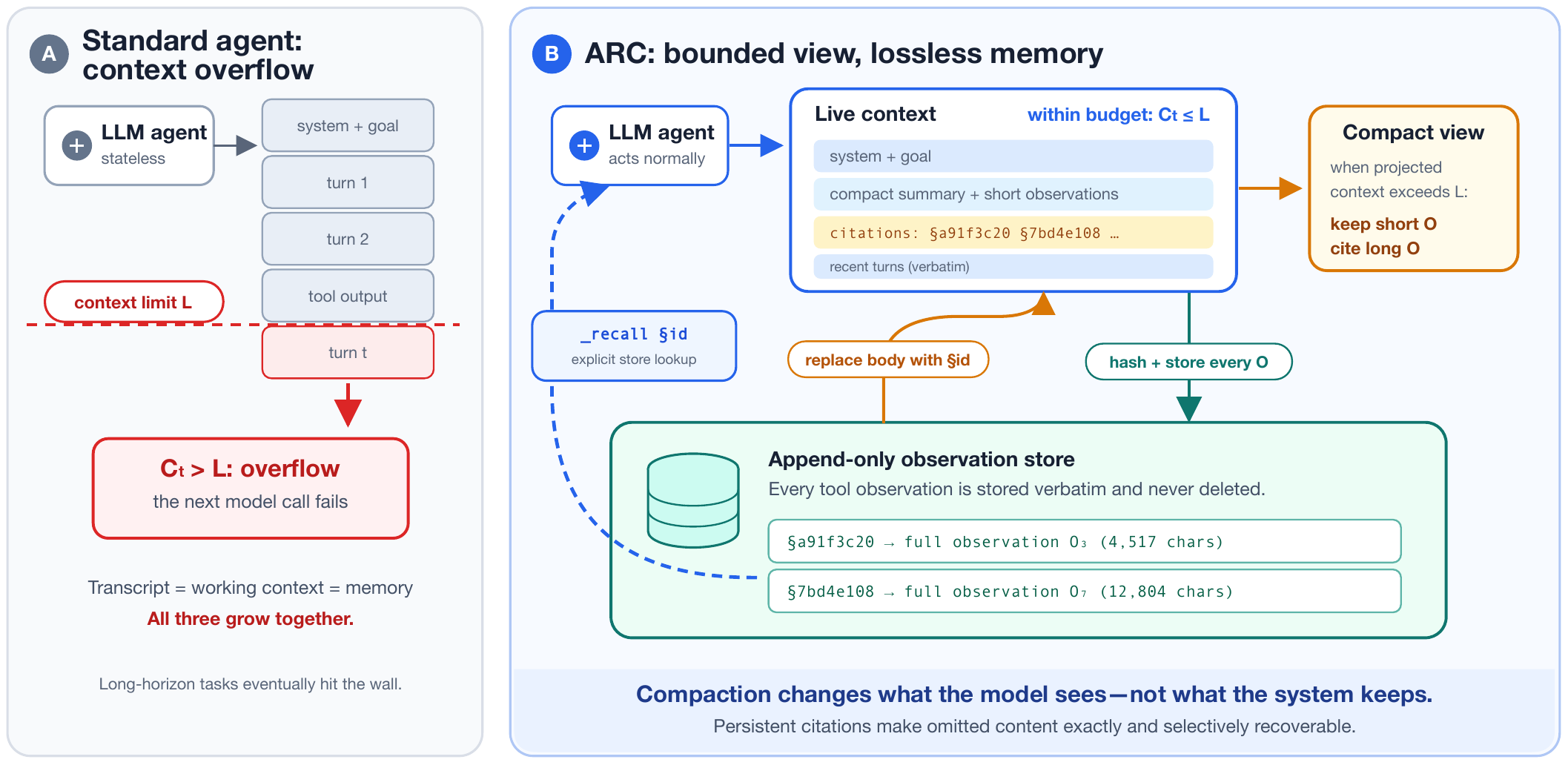}
\caption{Our Addressable Recall Compaction - ARC vs a standard agent context management.}
\label{fig:ARC_conceptual_workflow}
\end{figure*}

\begin{figure}[t]
\centering
\includegraphics[width=0.95\columnwidth]{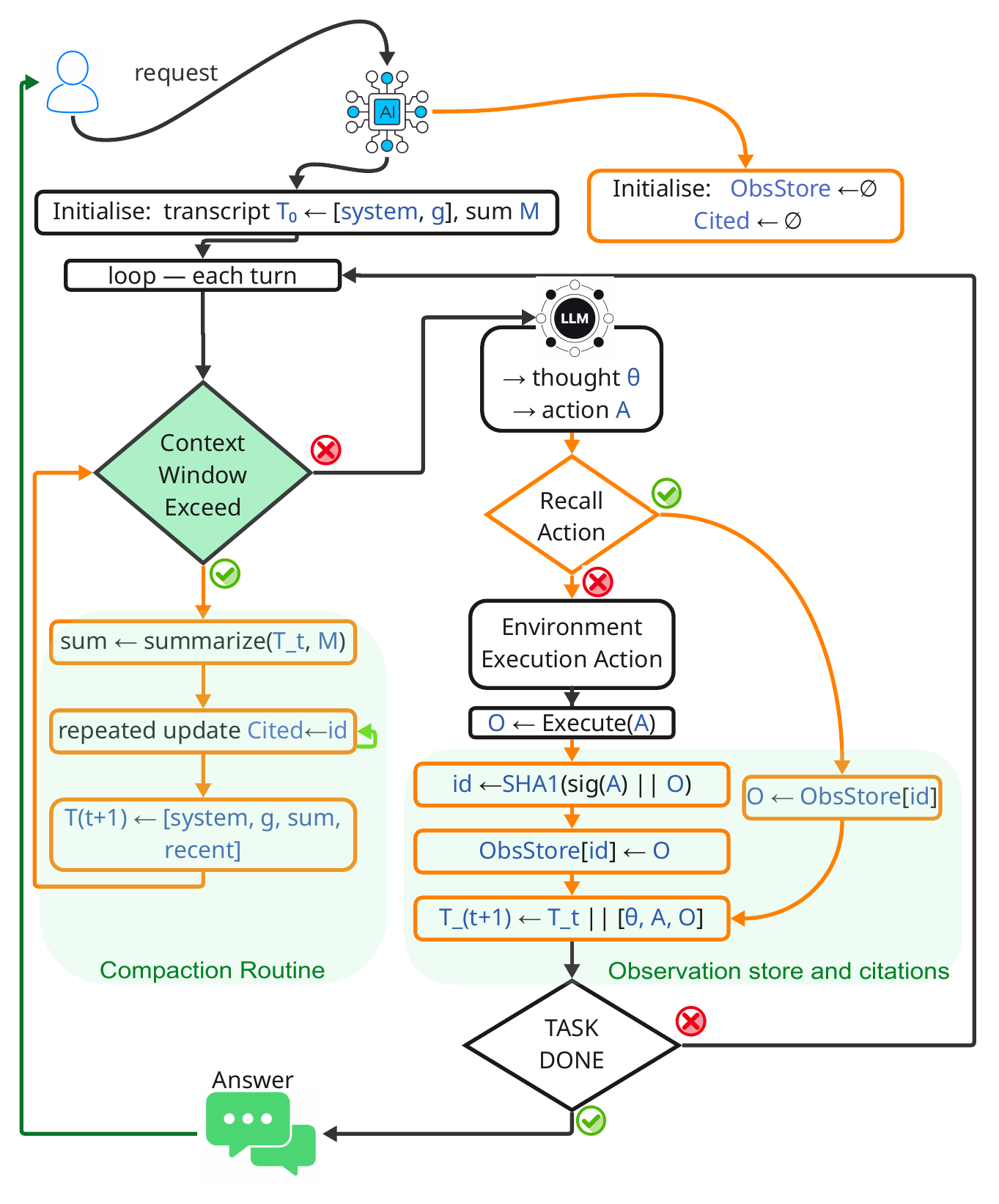}
\caption{Our Addressable Recall Compaction workflow.}
\label{fig:ARC_workflow}
\end{figure}

To prevent this catastrophic overflow, contemporary agent frameworks rely on compaction or compression~\cite{ichikawa2026onecomp}: a suite of lossy transformation techniques designed to compress an over-long transcript into a truncated form that fits comfortably within context-window length. The current state-of-the-art methods generally fall into four paradigms:

\begin{itemize}
    \item \textbf{Naive Truncation}~\cite{wang2026context}: Dropping the oldest $N$ turns entirely, assuming older context is irrelevant.
    \item \textbf{Hierarchical Summarization}~\cite{kang2510acon}: Periodically prompting the LLM to condense historical turns into high-level textual summaries.
    \item \textbf{State Folding}~\cite{fu2026vikingmem}: Mapping conversational histories into explicit, structured JSON or key-value state objects.
    \item \textbf{Retrieval-Augmented Generation (RAG)}~\cite{park2023generative}: Offloading historical context into an external vector database and retrieving relevant chunks via semantic similarity.
\end{itemize}

All four paradigms are fundamentally lossy: once history is pruned, paraphrased, or vectorized, exact recovery is no longer guaranteed. If the discarded content contains a critical \texttt{needle} such as a variable name, error string, or user constraint, the agent must either re-derive it at high cost or hallucinate a substitute, compounding errors over long trajectories.

We address this limitation by asking whether transcript compaction can remain strictly lossless while keeping the active prompt bounded. We introduce \textbf{ARC} (Addressable Recall Compaction) (Fig.~\ref{fig:ARC_conceptual_workflow}), which decouples archival storage from in-context representation. ARC maintains (1) an append-only \textbf{Addressable Store} that preserves historical turns and tool outputs verbatim, and (2) a \textbf{Bounded-Size Active View} that replaces older content with explicit pointers into that store. With recall tools that resolve these pointers on demand, ARC preserves access to any historical \texttt{needle} without exceeding the active context budget.

\subsubsection*{Summary of Contributions}
\begin{itemize}
    \item We characterize how lossy compaction fails in long-horizon LLM agent trajectories.
    \item We present ARC, a lossless compaction framework based on addressable pointers and on-demand recall.
    \item We show that ARC prevents context-overflow failures and outperforms summarization- and RAG-based baselines on long-context, tool-heavy tasks.
\end{itemize}

%% file: related_works.tex
\section{Related Works}
\label{sec:ARC_related}

Prior compaction methods differ in what they remove, when removal happens, and
whether discarded content can later be recovered. We organize them by this last
property, since it determines whether compaction is \emph{safe}: whether
task-critical information can survive the process.

\paragraph{Full-Context (No Compaction).}
Some agent systems simply append every turn to the transcript indefinitely, as
in ReAct and Reflexion \cite{yao2023react,shinn2023reflexion}. This preserves
perfect fidelity while context remains available, but it is not a true
compaction strategy: once $|T| > L$, the method fails.

\paragraph{Sliding Window.}
A common baseline keeps only the most recent $N$ turns and drops older ones
\cite{ding2026sliding}. This is simple and budget-compliant, but once a value
falls outside the window, it is lost permanently.

\paragraph{LLM Summary.}
Another approach replaces older turns with a natural-language summary produced
by the model \cite{packer2023memgpt}. This retains more semantic content than
truncation, but remains lossy: omitted or paraphrased details cannot be
recovered, and each compaction event incurs extra latency and cost.

\paragraph{Structured State.}
Instead of free text, prior work may compress history into a fixed-schema state
object such as JSON \cite{fu2026vikingmem}. Schemas can preserve anticipated
fields more reliably than open-ended summaries, but any unmodeled information
is still discarded irreversibly.

\paragraph{RAG Memory.}
Retrieval-based methods store past observations externally and re-inject
retrieved chunks into context when needed \cite{lewis2020retrieval}. Although
the store may retain everything, recall is only implicit: the agent cannot
request a specific prior item directly, and recovery depends on retrieval
quality.

\paragraph{Summary.}
Across these approaches, the key limitation is the absence of reliable,
addressable recall. Some methods discard information outright; others preserve
it only indirectly through summaries, schemas, or similarity-based retrieval.
ARC addresses this gap by storing observations verbatim under canonical
identifiers and replacing them in the live transcript with short citations,
allowing explicit recovery of exact prior content on demand.

%% file: background.tex
\section{Background}

A long-horizon LLM agent operates in a loop: at each turn it receives an
observation (a tool result, an error message, a user message), produces a
thought and an action, and appends all of it to a growing transcript
$T_t = (o_1, a_1, o_2, a_2, \ldots, o_t, a_t)$ that is re-sent to the model on every
subsequent call \cite{yao2023react}. The model itself is stateless between
calls; everything it ``remembers'' about the task is whatever fits in $T_t$.
 
This creates a hard constraint. Every model serves requests under a fixed
context budget $L$ (measured in tokens). Once $|T| > L$, the next call fails
outright, producing a context-window overflow, regardless of how important
the content being dropped is. Overflow is not a rare edge case: any task
that runs long enough (many tool calls, large tool outputs, extended
reasoning) will eventually exceed $L$, so a real agent must have
\emph{some} answer to this moment.
 
We call the general problem \textbf{compaction}: producing a shorter
transcript $T_t'$ from $T_t$ once $|T| > L$, such that the agent can continue
operating within budget. Every method surveyed in this paper is a compaction
strategy in this sense, and they differ along one central axis: what
happens to the information that no longer fits verbatim in $T_t'$. We say a
value is \textbf{recoverable} under a given strategy if there exists some
later point in the task at which the agent can retrieve that value's exact
original form; otherwise it is \textbf{lost}. A compaction strategy is
\emph{lossy} if it can produce $T'$ from which some value is unrecoverable by
construction (e.g., truncated away, paraphrased, or never stored outside the
live transcript). This distinction of lost versus recoverable, and if
recoverable, by what mechanism the agent actually retrieves the value, is
the organizing axis for the related work below, since it determines whether a
compaction event can silently destroy information the agent will later need,
independent of how fluent or efficient the compaction itself is.

To measure this directly, we adopt a needle-in-a-haystack~\cite{nelson2024needle, kamradt2023needle} protocol: a single random, unforgeable value (the needle) is introduced early in a task, the
agent is driven through one or more forced compaction events, and the task is
scored by whether the agent can still produce the needle's exact value when
asked. This isolates recoverability from general task competence, since a
method can compact skillfully and still fail this test if the mechanism it
uses provides no reliable path back to the needle once it has left the live
transcript.

%% file: proposed_method.tex
\section{Proposed Method}
\label{sec:proposed_method}

\subsection{Design Principle}
Every baseline in the Related Works Section conflates two decisions that need not be coupled: (1) what to keep, and (2) what to show the model right now. ARC separates them. It keeps \emph{everything}, always, in a content-addressed store; it decides what to show using the same length-bounded logic as the other baselines, but every omission leaves behind an explicit, dereferenceable pointer. Recovery is then a decision that the agent makes, not a probability that the retriever gets right.

\subsection{ARC Algorithm (Fig.~\ref{fig:ARC_workflow})}
\label{alg:proposed_method}

\textbf{Observation store and citations.}
For every normal action $A$, ARC stores the resulting observation $O$ under a stable id computed as $\text{Hash}(\text{signature}(A), O)$: concretely, SHA1 over the action signature, an ASCII unit-separator byte, and the observation. The visible id uses the first 8 hex characters and is extended in 4-character increments only if two different observations collide. Each \texttt{ObsStore} record keeps the full content, command signature, fixed head/tail previews, length metadata, return code, creation step, and recall bookkeeping. The store is append-only for the lifetime of the task.

When compaction removes a long observation from the visible transcript, ARC leaves a citation \texttt{§id} containing its fixed head preview, tail preview, metadata, and a \texttt{\_recall §id} hint. Observations shorter than $\rho$ remain verbatim. The set \texttt{Cited} enforces citation persistence: once an id has appeared as a recall handle, later compactions preserve that handle rather than re-summarizing or renaming it.

\textbf{Recall action.}
\texttt{\_recall §id} is intercepted by the agent framework before it reaches the environment. A valid recall injects \texttt{ObsStore[id]} back into the transcript; an invalid id returns a nearest-match suggestion instead of failing silently. Recalls are controlled by a per-step limit $r$ (default $r=2$), a total materialized-recall budget, and a per-recall maximum length; least-recently-used recalled bodies are evicted back to citation stubs when the budget is exceeded. A lighter \texttt{\_recall\_meta §id} action restores only citation-level metadata.

\textbf{Compaction routine.}
$\text{Summarize}$ is deterministic and does not call an LLM, unlike the LLM-summary method, this step adds no extra model invocation or latency. It truncates older thoughts, keeps short observations inline, replaces longer observations by citations, and retains the most recent turns verbatim. If the assembled summary still exceeds $M$, older already-cited entries are collapsed from full citation stubs to bare \texttt{§id} references. These operations change only the active view; the full bodies remain recoverable from \texttt{ObsStore}.

\textbf{Triggering.}
ARC supports both proactive and reactive triggering. In proactive mode, the agent estimates prompt length before each model call and compacts before overflow. In our real-LLM evaluation, we use the reactive path: compaction occurs after a context-window-overflow error, with fixed call indices used to isolate post-compaction recovery from stochastic overflow timing. Both modes invoke the same deterministic compaction routine and recall mechanism.

The key invariant is that after compaction, the visible transcript has a bounded size $K \le L$ independent of the number of elapsed turns, while \texttt{ObsStore} retains every observation at full fidelity. ARC, therefore, shrinks the model's view without shrinking the task's ground truth.

\subsection{System-Level Guarantee}
\label{sec:arc_informal_guarantee}

The guarantee we establish for ARC is a property of the memory mechanism, not of the agent that uses it. It ensures that any stored observation remains exactly recoverable through a valid address, but it does not ensure that the language model will choose the correct address or reason accurately once the content is retrieved. We therefore draw a deliberate distinction between address-conditioned recoverability, which the mechanism provides, and end-to-end task success, which the agent must additionally achieve. In the same spirit, we use persistent addressability to mean that every issued address remains valid in the external catalog for the lifetime of the task. This is a weaker and more realistic requirement than keeping all citations simultaneously visible in the active prompt, which, as we show in the appendix, is incompatible with a turn-independent context bound.

Let $L$ denote the usable input-token budget, defined as the context limit minus the capacity reserved for the model's maximum completion length. Within the active view, ARC assigns a fixed token budget to each of the five components: the immutable task prefix, the deterministic summary, the recent suffix, a single page of the external citation catalog, and the recalled content. Let $K$ be the sum of these budgets, and require $K \leq L$. Crucially, ARC never recovers a large observation through head--tail truncation, which would irreversibly discard interior tokens; instead, it returns the observation through exact, non-overlapping recall chunks of at most $q$ payload tokens each, ensuring that the full content is always reconstructible on demand.

\begin{arctheorem}[Informal ARC guarantee]
\label{thm:arc_informal}
    Suppose ARC (i) assigns injective, fixed-width occurrence identifiers up to a declared task capacity, (ii) stores every tool observation verbatim before any compaction can remove it, (iii) keeps the complete citation catalog in external memory while materializing only a token-bounded page in the prompt, and (iv) enforces the active-view budget $K \leq L$ before every model invocation. Then, at every turn, the model input contains at most $K$ tokens, independently of how many turns have elapsed. Moreover, for every stored observation $O$ with valid identifier $\mathrm{id}$, ARC returns the exact archived token sequence without re-executing the original action; if $|O| > q$, the full sequence is reconstructed from exactly $\left\lceil \nicefrac{|O|}{q} \right\rceil$ non-overlapping recall responses, each carrying at most $q$ payload tokens. Consequently, ARC is observation-lossless at the system level while maintaining a uniformly bounded active view.
\end{arctheorem}

We emphasize that this guarantee is intentionally conditioned on a valid address: it establishes exact availability and bounded presentation, but not that the model will request the correct identifier. The appendix provides the self-contained definitions and full proofs. It further shows that this design overhead is essentially unavoidable: any method that supports the exact recovery of arbitrary histories while keeping the active state bounded must use external memory that grows linearly in the worst case. ARC's append-only payload growth, therefore, matches this lower bound to first order rather than reflecting a limitation of our particular implementation.

%% file: experimental_setup.tex
\section{Experimental Setup}
\label{sec:exp_setup}
\subsection{Models Configuration}
\begin{table*}[htbp]
\centering
\resizebox{\textwidth}{!}{
\begin{tabular}{llcccrcccc}
\toprule
\rowcolor{gray!12}
\textbf{Method} & \textbf{Model} & \textbf{Ctx.} & \textbf{Max-T.} & \textbf{Pass ↑/ Attempts} & \textbf{Mean No-Ans}↓ & \textbf{Mean Pass}↑ & \textbf{Mean Acc}↑ & \textbf{SD Acc} & \textbf{Mean $\pm$ SD} \\
\midrule
\midrule
 Full\_context & \multirow{6}{*}{\textbf{Qwen3-8B}} & 16k & 4096 & 0 / 933 & 311 & 0 & 0.0\% &0.0\% & 0.0\% $\pm$ 0.0\% \\
 Sliding\_window  & & -- & -- & 219 / 933 & 48.67 & 73 & 23.47\% & 1.15\% & 23.47\% $\pm$ 1.15\% \\
LLM\_summary & & -- & -- & 241 / 933 & 33 & 80.33 & 25.83\% & 0.96\% & 25.83\% $\pm$ 0.96\% \\
 Structured\_state && -- & -- & 231 / 933 & 40 & 77 & 24.73\% & 1.94\% & 24.73\% $\pm$ 1.94\% \\
 RAG\_memory && -- & -- & 229 / 933 & 45 & 76.33 & 24.53\% & 2.70\% & 24.53\% $\pm$ 2.70\% \\
 \textbf{ARC} && -- & -- & \textbf{256} / 933 & \textbf{16.33} & \textbf{85.33} & \textbf{27.47\%} & 1.69\% & \textbf{27.47\% $\pm$ 1.96\%} \\
\midrule
 Full\_context & \multirow{6}{*}{\textbf{Qwen3-32B}} & 32k & 4096 & 0 / 933 & 311 & 0 & 0\% & 0.0\% & 0.0\% $\pm$ 0.0\% \\
  Sliding\_window && -- & -- & 286 / 933 & 23.67 & 95.33 & 30.67\% & 0.68\% & 30.67\% $\pm$ 0.68\% \\
 LLM\_summary && -- & -- & 233 / 933 & 73.67 & 77.67 & 24.97\% & 2.28\% & 24.97\% $\pm$ 2.28\% \\
  structured\_state && -- & -- & 282 / 933 & 27 & 94 & 30.2\% & 0.30\% & 30.2\% $\pm$ 0.30\% \\
  RAG\_memory && -- & -- & 256 / 933 & 58 & 85.33 & 27.43\% & 0.75\% & 27.43\% $\pm$ 0.75\% \\
  \textbf{ARC} && -- & -- & \textbf{303} / 933 & \textbf{7.33} & \textbf{101} & \textbf{32.47\%} & 1.42\% & \textbf{32.47\% $\pm$ 1.42\%} \\
\bottomrule
\end{tabular}
}
\caption{LongBench-v2 hard-subset accuracy, 3 seeds (42, 82, 122) of 311 tasks each (933 attempts/method).}
\label{tab:longbench_results}
\end{table*}

We evaluate Qwen3-8B and Qwen3-32B served with vLLM, using 16k and 32k context windows respectively and a 4k completion cap. The context window is the vLLM limit on prompt plus completion length; when a trajectory exceeds it, vLLM raises a context-window error and the agent invokes its compaction path. Thus each method is compared under the same model scale, context budget, and per-call generation cap.

\subsection{Benchmarks}

\begin{table*}[htbp]
\centering
\resizebox{\textwidth}{!}{
\begin{tabular}{llcccrcccc}
\toprule
\rowcolor{gray!12}
\textbf{Method} & \textbf{Model} & \textbf{Ctx.} & \textbf{Max-T.} & \textbf{Pass↑ / Attempts} & \textbf{Mean No-Ans}↓ & \textbf{Mean Pass}↑ & \textbf{Mean Acc}↑ & \textbf{SD Acc} & \textbf{Mean $\pm$ SD} \\
\midrule
\midrule
 Full\_context & \multirow{6}{*}{\textbf{Qwen3-8B}} & 16k & 4096 & 1104 / 3000 & 631.67 & 368.00 & 36.80\% & 2.59\% & 36.80\% $\pm$ 2.59\% \\
 Sliding\_window  & & -- & -- & 1686 / 3000 & 437.33 & 562.00 & 56.20\% & 0.90\% & 56.20\% $\pm$ 0.90\% \\
LLM\_summary & & -- & -- & 2153 / 3000 & 279.33 & 717.67 & 71.77\% & 1.36\% & 71.77\% $\pm$ 1.36\% \\
 Structured\_state && -- & -- & 2045 / 3000 & 314.33 & 681.67 & 68.17\% & 1.64\% & 68.17\% $\pm$ 1.64\% \\
 RAG\_memory && -- & -- & 2387 / 3000 & 201.00 & 795.67 & 79.57\% & 0.80\% & 79.57\% $\pm$ 0.80\% \\
 \textbf{ARC} && -- & -- & \textbf{2970} / 3000 & \textbf{9.67} & \textbf{990.00} & \textbf{99.00\%} & 0.17\% & \textbf{99.00\% $\pm$ 0.17\%} \\
\midrule
 Full\_context & \multirow{6}{*}{\textbf{Qwen3-32B}} & 32k & 4096 & 88 / 3000 & 964.67 & 29.33 & 2.93\% & 0.76\% & 2.93\% $\pm$ 0.76\% \\
  Sliding\_window && -- & -- & 2860 / 3000 & 44.33 & 953.33 & 95.33\% & 0.64\% & 95.33\% $\pm$ 0.64\% \\
 LLM\_summary && -- & -- & 2817 / 3000 & 22.33 & 939.00 & 93.90\% & 0.52\% & 93.90\% $\pm$ 0.52\% \\
  Structured\_state && -- & -- & 2899 / 3000 & 15.33 & 966.33 & 96.63\% & 0.78\% & 96.63\% $\pm$ 0.78\% \\
  RAG\_memory && -- & -- & 2900 / 3000 & 12.33 & 966.67 & 96.67\% & 1.15\% & 96.67\% $\pm$ 1.15\% \\
  \textbf{ARC} && -- & -- & \textbf{2994} / 3000 & \textbf{0.67} & \textbf{998.00} & \textbf{99.80\%} & 0.17\% & \textbf{99.80\% $\pm$ 0.17\%} \\
\bottomrule
\end{tabular}
}
\caption{Needle-in-a-haystack Task, 3 seeds (42, 82, 122) of 1000 tasks each (3000 attempts/method).}
\label{tab:Needle_results}
\end{table*}

\textbf{Evaluation metrics.}
We report accuracy and, separately, \texttt{no\_ans}: the number of trials with no accepted submission. A submission counts only after the harness protocol is completed; premature, blocked, crashed, or missing submissions are counted as failures. Since non-submissions are already incorrect, \texttt{no\_ans} is diagnostic rather than a separate success metric: it distinguishes wrong answers from failures to reach an answer at all.

\textbf{HBM bandwidth and roofline metric~\cite{scaling-book}.}
Context management changes not only task success but also serving cost. We therefore replay each method's recorded prompt/completion trajectory through a fixed hardware model of one NVIDIA H200 143GB GPU and estimate prefill, decode, and total time. During autoregressive decode, each generated token streams the model weights $M_w$, amortized over the concurrent batch size $N_b$, plus the request's own KV cache of length $L$:
\begin{equation}
\text{bytes per decoded token (per request)} = \nicefrac{M_w}{N_b} + L \cdot \kappa,
\end{equation}
where $\kappa$ is the per-token KV-cache footprint. Decode time is bytes streamed divided by HBM bandwidth, while prefill time is prefill FLOPs divided by peak Tensor Core throughput. We report $T_{\text{total}} = T_{\text{prefill}} + T_{\text{decode}}$ as a roofline estimate of serving cost.

The attainable batch size is limited by the HBM remaining after model weights and a fixed scratch reserve $S$:
\begin{equation}
N_b = \left\lfloor \frac{n_{\text{gpu}} \cdot \text{HBM}_{\text{gpu}} - M_w - S}{\bar L \cdot \kappa} \right\rfloor,
\end{equation}
where $\bar L$ is the decode-token-weighted mean context length. Shorter active contexts therefore reduce the KV term directly and increase $N_b$, reducing the amortized weight term $M_w/N_b$. In the Results section, we abbreviate roofline quantities as \textbf{t\_tol.} (total time), \textbf{t\_dec.} (decode time), \textbf{t\_pre.} (prefill time), \textbf{bw} (bandwidth), and \textbf{Sav.} (savings).

\textbf{Paired significance testing.}
Because all methods use the same seeded tasks/needles, outcomes are paired. For each ARC-vs-baseline comparison, we use McNemar's test on the discordant cells $a_{\text{only}}$ (ARC succeeds, baseline fails) and $b_{\text{only}}$ (baseline succeeds, ARC fails), reporting the continuity-corrected statistic
\begin{equation}
\chi^2 = \frac{(|a_{\text{only}} - b_{\text{only}}| - 1)^2}{a_{\text{only}} + b_{\text{only}}},
\end{equation}
and an exact binomial $p$-value over discordant pairs. We apply this test to correctness (\texttt{success} vs. failure) and to completion (accepted submission vs. \texttt{no\_ans}), separating answer quality from the ability to finish the protocol.

\textbf{LongBench-v2 (Hard Subset)}
We evaluate on the LongBench-v2 hard subset~\cite{bai2025longbench}, containing 311 tasks across single-document QA, multi-document QA, long in-context learning, structured-data understanding, long-dialogue-history understanding, and code-repository understanding. Unlike the needle benchmark, LongBench-v2 requires multi-step reasoning over long inputs, so it tests whether ARC's recall mechanism helps beyond verbatim retrieval.

\vspace{1em}
\textbf{Needle-in-a-Haystack (NIH)~\cite{gao2026u}}
The benchmark needs a single, unambiguous, unforgeable fact that the agent must retain across compaction events. That fact is the needle, a 24-character token drawn from a 30-character alphabet, embedded inside a synthetic file as \texttt{secrets.txt}.

\vspace{1em}

\textbf{Needle benchmark guardrail.}
NIH evaluations can be confounded by self-doubt~\cite{gao2026u}: a model may find a random token but hesitate to submit it. Our harness reduces this risk by making the needle a banner-delimited instruction labeled \texttt{THE NEEDLE FOR THIS TASK}, marking distractors as noise, and scoring only the submitted \texttt{NEEDLE=<value>} token rather than surrounding prose. All methods also use the same seeded needle pool, so residual per-needle difficulty is paired across methods.

%% file: results.tex
\section{Results}
\label{sec:results}

We evaluate ARC on the LongBench-v2 hard subset with Qwen3-8B and Qwen3-32B. 
As shown in Table~\ref{tab:longbench_results}, ARC achieves the best accuracy at both scales: 27.47\% on 8B and 32.47\% on 32B, improving over the strongest baselines by +1.64 and +1.80 points, respectively.

ARC also sharply reduces unanswered cases, with mean No-Answer rates of 16.33 on 8B and 7.33 on 32B, while `Full\_context` fails completely at both context limits. Unlike static compression methods, ARC scales favorably to the larger backbone and remains stable across seeds.

Table~\ref{tab:HBM_longbench} shows that ARC's gains are accompanied by markedly lower HBM traffic. On LongBench, ARC reduces bandwidth by 38.8\% on Qwen3-8B and 73.5\% on Qwen3-32B relative to `Sliding\_window', corresponding to substantially lower wall-clock time. ARC is the only method that improves task success while also reducing memory traffic and decoding cost.

\begin{table}[htbp]
\centering
\setlength{\tabcolsep}{4pt}
\begin{tabular}{l|rrrrr}
\toprule
\rowcolor{gray!12}
\textbf{Method}   & \textbf{t\_tot.} & \textbf{t\_dec.} & \textbf{t\_pre.} & \textbf{bw} & \textbf{Sav.}  \\
\rowcolor{gray!12}
  & \textbf{(s) ↓} & \textbf{(s) ↓} & \textbf{(s) ↓} & \textbf{↓} & \textbf{(\%)↑} \\
\midrule
\midrule
\rowcolor{blue!20}
\multicolumn{6}{l}{\textbf{Qwen3-8B}} \\
Sliding\_window & 6.53 & 5.76 & 0.77 & 27.63 & 0.0  \\
LLM\_summary   & 8.11 & 7.47 & 0.64 & 35.84 & -24.3 \\
Structured\_state    & 9.68 & 9.07 & \textbf{0.61} & 43.53 & -9.8  \\
RAG\_memory    & 14.68 & 13.85 & 0.83 & 66.48 & -124.9  \\
\textbf{ARC}   & \textbf{3.99} & \textbf{3.32} & 0.67 & \textbf{15.95} & \textbf{38.8}  \\
\midrule
\rowcolor{blue!20}
\multicolumn{6}{l}{\textbf{Qwen3-32B}} \\
Sliding\_window    & 30.57 & 28.76 & 1.81 & 138.05 & 0.0  \\
LLM\_summary   & 75.5 & 72.45 & 3.05 & 347.73 & -146.9  \\
Structured\_state    & 28.67 & 27.25 & \textbf{1.42} & 130.79 & 6.2  \\
RAG\_memory   & 77.01 & 74.55 & 2.46 & 357.83 & -151.9  \\
\textbf{ARC}  & \textbf{8.11} & \textbf{6.44} & 1.68 & \textbf{30.9} & \textbf{73.5} \\
\bottomrule
\end{tabular}
\caption{High-Bandwidth Memory (HBM) comparison on success submitted Longbench tasks.}
\label{tab:HBM_longbench}
\end{table}

\begin{table}[htbp]
\centering
\setlength{\tabcolsep}{4pt}
\begin{tabular}{l|rrrrr}
\toprule
\rowcolor{gray!12}
\textbf{Method}   & \textbf{t\_tot.} & \textbf{t\_dec.} & \textbf{t\_pre.} & \textbf{bw} & \textbf{Sav.}  \\
\rowcolor{gray!12}
  & \textbf{(s) ↓} & \textbf{(s) ↓} & \textbf{(s) ↓} & \textbf{↓} & \textbf{(\%)↑} \\
\midrule
\midrule
\rowcolor{blue!20}
\multicolumn{6}{l}{\textbf{Qwen3-8B}} \\
Sliding\_window & 6.38 & 5.77 & 0.61 & 27.70 & 0.0  \\
LLM\_summary   & 9.49 & 8.95 & 0.54 & 42.97 & -48.9 \\
Structured\_state    & 10.59 & 10.12 & 0.47 & 48.55 & -66.0  \\
RAG\_memory    & 6.47 & 6.03 & 0.44 & 28.92 & -1.4  \\
\textbf{ARC}   & \textbf{1.26} & \textbf{0.01} & \textbf{0.24} & \textbf{4.87} & \textbf{80.3}  \\
\midrule
\rowcolor{blue!20}
\multicolumn{6}{l}{\textbf{Qwen3-32B}} \\
Sliding\_window    & 0.83 & 0.58 & 0.26 & 2.76 & 0.0  \\
LLM\_summary   & 0.85 & 0.64 & 0.2 & 3.09 & -77.4  \\
Structured\_state    & 0.67 & 0.51 & \textbf{0.16} & 2.46 & 33.1  \\
RAG\_memory   & 0.64 & 0.47 & 0.17 & 2.25 & 11.9  \\
\textbf{ARC}  & \textbf{0.55} & \textbf{0.33} & 0.22 & \textbf{1.59} & \textbf{59.3} \\
\bottomrule
\end{tabular}
\caption{High-Bandwidth Memory (HBM) comparison on success submitted Needle-in-a-haystack tasks.}
\label{tab:HBM_needle}
\end{table}

\begin{table*}[htbp]
\centering
\resizebox{\textwidth}{!}{%
\begin{tabular}{l|l||cccccc ccc}
\toprule
\rowcolor{gray!12}
\textbf{A} & \textbf{B}  & \textbf{A-Only} & \textbf{B-Only} & \textbf{Both} & \textbf{Neither} & \textbf{$\chi^2$ p-val} & \textbf{Exact p-val} & \textbf{Odds Ratio} & \textbf{Per-Seed Odds ($\pm$ SD)} & \textbf{Dis. Seeds} \\
\midrule
\midrule
\rowcolor{blue!20}
\multicolumn{11}{l}{\textbf{Qwen3-8B}} \\
ARC & Full\_context  & 884 & 0 & 0 & 49 & 0 & $7.75 \times 10^{-267}$  & inf & inf & 0 \\
ARC & LLM\_summary  & 91 & 41 & 793 & 8 & $2.0 \times 10^{-5}$ & $8.0 \times 10^{-6}$  & 2.22 & $2.5 \pm 1.38$ & 3 \\
ARC & RAG\_memory  & 123 & 37 & 761 & 12 & 0 & $2.73 \times 10^{-12}$  & 3.32 & $3.88 \pm 2.21$ & 3 \\
ARC & Sliding\_window   & 138 & 41 & 746 & 8 & 0 & $8.88 \times 10^{-14}$ & 3.37 & $3.97 \pm 2.2$ & 3 \\
ARC & Structured\_state  & 113 & 42 & 771 & 7 & 0 & $5.24 \times 10^{-9}$  & 2.69 & $3.27 \pm 2.01$ & 3 \\
\midrule
\rowcolor{blue!20}
\multicolumn{11}{l}{\textbf{Qwen3-32B}} \\
ARC & Full\_context  & 911 & 0 & 0 & 22 & 0 & $5.78 \times 10^{-275}$  & inf & inf & 0 \\
ARC & LLM\_summary   & 220 & 21 & 691 & 1 & 0 & $2.63 \times 10^{-43}$ & 10.48 & $10.57 \pm 1.03$ & 3 \\
ARC & RAG\_memory   & 173 & 21 & 738 & 1 & 0 & $3.19 \times 10^{-31}$ & 8.24 & $8.27 \pm 0.56$ & 3 \\
ARC & Sliding\_window  & 68 & 19 & 843 & 3 & 0 & $6.18 \times 10^{-8}$  & 3.58 & $3.71 \pm 1.12$ & 3 \\
ARC & Structured\_state   & 79 & 20 & 832 & 2 & 0 & $8.96 \times 10^{-10}$ & 3.95 & $4.07 \pm 1.05$ & 3 \\
\bottomrule
\end{tabular}
}
\caption{McNemar test on Longbench task. Number of Pairs (hard task) = 311x3 seeds (42, 82, 122).}
\label{tab:mcNemar_Longbench_comparison}
\end{table*}

To rigorously assess whether the performance gains of our proposed method, ARC, are statistically significant, we conduct a McNemar's test on the matched-paired binary outcomes across all 933 evaluation attempts (311 tasks $\times$ 3 seeds). McNemar's test focuses exclusively on the discordant pairs---representing the cases where only one of the two compared methods succeeds. The results are detailed in Table~\ref{tab:mcNemar_Longbench_comparison}. On Qwen3-8B, ARC yields odds ratios from 2.22 against `LLM\_summary' to 3.37 against `Sliding\_window', showing a clear advantage on discordant cases. The gap widens on Qwen3-32B: ARC reaches odds ratios of 10.48 against `LLM\_summary' and 8.24 against `RAG\_memory', while still outperforming the strongest larger-model baselines, `Sliding\_window' and `Structured\_state', with odds ratios of 3.58 and 3.95. Since `Full\_context' records zero successes at both scales, ARC also attains infinite odds ratios in those comparisons.

Table~\ref{tab:Needle_results} shows near-perfect retrieval performance for ARC on Needle-in-a-Haystack. ARC reaches 99.00\% accuracy on Qwen3-8B and 99.80\% on Qwen3-32B, exceeding the strongest baseline, `RAG\_memory', by 19.43 and 3.13 points, respectively, while nearly eliminating unanswered queries.

ARC also has the lowest variance across seeds (0.17\% for both models). In contrast, `Full\_context' remains unreliable even with a 32k window, reaching only 2.93\% on Qwen3-32B, indicating that larger context alone does not solve precise retrieval.

Table~\ref{tab:mcNemar_needle_comparison} confirms that these retrieval gains are statistically significant in every paired comparison ($p \ll 0.001$). On Qwen3-8B, ARC dominates all baselines, including `RAG\_memory' (599 ARC-only wins vs. 25 baseline-only wins; odds ratio 23.96). On Qwen3-32B, ARC remains significantly better despite the stronger baselines, with an odds ratio of 18.5 against `RAG\_memory' and an infinite odds ratio against `Full\_context'.

\begin{table*}[htbp]
\centering
\resizebox{\textwidth}{!}{%
\begin{tabular}{l|l||c cccc cc ccc}
\toprule
\rowcolor{gray!12}
\textbf{A} & \textbf{B}  & \textbf{A-Only} & \textbf{B-Only} & \textbf{Both} & \textbf{Neither} & \textbf{$\chi^2$ p-val} & \textbf{Exact p-val} & \textbf{Odds Ratio} & \textbf{Per-Seed Odds ($\pm$ SD)} & \textbf{Dis. Seeds} \\
\midrule
\midrule
\rowcolor{blue!20}
\multicolumn{11}{l}{\textbf{Qwen3-8B}} \\
\multirow{5}{*}{\textbf{ARC}}& Full\_context & 1882 & 16 & 1089 & 13 & 0 & 0 & 117.62 & $123.43 \pm 30.41$ & 3 \\
 & LLM\_summary & 831 & 22 & 2140 & 7 & 0 & $3.5 \times 10^{-214}$ & 37.77 & $40.32 \pm 13.76$ & 3 \\
 & RAG\_memory  & 599 & 25 & 2372 & 4 & 0 & $4.5 \times 10^{-144}$ & 23.96 & $26.42 \pm 10.13$ & 3 \\
 & Sliding\_window  & 1300 & 17 & 1671 & 12 & 0 & 0 & 76.47 & $89.72 \pm 45.54$ & 3 \\
 & Structured\_state& 934 & 20 & 2037 & 9 & 0 & $8.8 \times 10^{-247}$ & 46.7 & $52.87 \pm 21.05$ & 3 \\
\midrule
\rowcolor{blue!20}
\multicolumn{11}{l}{\textbf{Qwen3-32B}} \\
\multirow{5}{*}{\textbf{ARC}} & full\_context   & 2892 & 0 & 106 & 2 & 0 & 0 & inf & inf & 0 \\
 & LLM\_summary  & 67 & 2 & 2931 & 0 & 0 & $4.09 \times 10^{-18}$ & 33.5 & $9.5 \pm 0.0$ & 1 \\
 & RAG\_memory   & 37 & 2 & 2961 & 0 & 0 & $1.42 \times 10^{-9}$ & 18.5 & $5.5 \pm 0.0$ & 1 \\
 & Sliding\_window   & 133 & 2 & 2865 & 0 & 0 & $2.11 \times 10^{-37}$ & 66.5 & $20.0 \pm 0.0$ & 1 \\
 & Structured\_state  & 46 & 2 & 2952 & 0 & 0 & $4.18 \times 10^{-12}$ & 23.0 & $9.5 \pm 0.0$ & 1 \\
\bottomrule
\end{tabular}
}
\caption{McNemar test on Needle-in-a-haystack task. Number of Pairs = 1,000x3 seeds (42, 82, 122).}
\label{tab:mcNemar_needle_comparison}
\end{table*}

Table~\ref{tab:HBM_needle} shows that ARC improves efficiency as well as retrieval quality. On Qwen3-8B, ARC raises successful retrievals to 991 while reducing total inference time from 6.47 s (`RAG\_memory') to 1.26 s and cutting HBM traffic by 80.3\%. On Qwen3-32B, ARC again achieves the highest success count (999), the lowest latency (0.55 s), and the lowest bandwidth use (1.59 TB). These results indicate that ARC improves accuracy without the usual latency-bandwidth tradeoff.

Overall, these results show that ARC improves both retrieval effectiveness and computational efficiency. Unlike existing memory management methods that often trade latency for accuracy, ARC simultaneously achieves near-perfect retrieval performance while significantly reducing decoding cost and HBM bandwidth requirements, making it well suited for efficient long-context inference.

%% file: discussion_limittation.tex
\section{Discussion and Limitations}
\label{sec:discussion}

\paragraph{Why the margin narrows on LongBench-v2.}
Needle-in-a-haystack isolates verbatim recall---exactly what ARC targets---while LongBench-v2 hard tasks additionally require synthesizing information across a long document, so perfect recall can still yield a wrong answer through faulty reasoning. On LongBench-v2 (Table~\ref{tab:longbench_results}), ARC still leads every baseline at both scales (27.47\% at Qwen3-8B, 32.47\% at Qwen3-32B), but the margin over the closest competitor narrows to 1.6--2.3 points (LLM\_summary at 8B, sliding\_window at 32B), versus 4--7.5 points over the weaker baselines. ARC's advantage is therefore better interpreted as a lower bound on how much of LongBench's failure rate is attributable to recall, further corroborated by its markedly lower no-answer rate (16.33 and 7.33 out of 311, versus 23--74 for every baseline)---rather than the full ceiling of ARC's benefit.

\paragraph{Statistical significance.}
Per-trial McNemar tests (Tables~\ref{tab:mcNemar_needle_comparison}, \ref{tab:mcNemar_Longbench_comparison}) confirm ARC significantly outperforms every baseline on both benchmarks and both model scales ($p<0.05$). Against `Full\_context', ARC wins every one of the 933 paired trials at both 8B and 32B, with zero baseline-only wins. Against the remaining baselines, \texttt{b\_only} counts (baseline succeeds, ARC fails) stay small relative to \texttt{a\_only} 37--42/933 at 8B and 19--21/933 at 32B---yielding odds ratios from 2.2 up to 10.5 in ARC's favor; the effect is strongest for `LLM\_summary' at 32B (odds ratio 10.48) and weakest, though still significant, for `Sliding\_window' and `Structured\_state', whose closer overall accuracy reflects that reasoning errors, unlike recall failures, can occasionally favor a baseline over ARC on individual tasks.

\paragraph{Scope and overhead.}
All experiments use the Qwen3 model family; generalization to other tokenizers, reasoning styles, or instruction-following behavior for the \texttt{\_recall} convention remains open. Each ARC citation (head preview, tail preview, recall hint) also adds a small, fixed token overhead relative to discarding an observation outright; quantifying this overhead and its effect on the effective context budget is reported in the supplementary document.

%% file: conclusion.tex
\section{Conclusion}
\label{sec:conclusion}

Lossy compaction trades context-window survival for a chance of losing exactly the fact an agent needs later, and every mainstream compaction strategy makes this trade. ARC shows the trade is not necessary: by separating an append-only, content-addressed store of everything the agent has seen from a bounded-size, citation-annotated view of it, compaction becomes reversible without sacrificing the constant-size guarantee that made truncation attractive in the first place. Across two representative configurations spanning both model sizes (Qwen3-8B/32B, both at a 4,096-token completion cap), this closes nearly all of the recovery gap on a benchmark built to isolate recall failure (99.40\% vs. 88.12\% for the best prior method) and yields a smaller but consistent gain on a benchmark that also demands reasoning (29.97\% vs. 28.25\%). The main open questions are quantifying ARC's own token overhead, testing generalization beyond the Qwen3 family, extending the evaluation to the remaining points of the model/context-window/max-token grid, and understanding why implicit retrieval still wins on dialogue-history tasks — a signal that explicit and implicit recall are complementary rather than competing mechanisms.

%% file: proof.tex
\section{Formal Guarantees for Addressable Recall Compaction}
\label{app:arc_guarantees}

This appendix formalizes three statements: (i) why keeping every citation simultaneously visible is incompatible with a turn-independent prompt bound, (ii) why a paginated ARC implementation provides a bounded active view and exact address-conditioned recall, and (iii) why linear external-memory growth is unavoidable for worst-case exact recovery.

\subsection{Formal Model}

\begin{arcdefinition}[Token sequences and observation history]
\label{def:observation_history}
Let $\Sigma$ be a finite token alphabet, and let $\Sigma^{\star}$ denote the set of finite token sequences. For $x \in \Sigma^{\star}$, let $|x|$ denote its token length. At turn $i$, let
$
\mathcal{O}_i = \left(O_1,O_2,\ldots,O_{N_i}\right)
$
be the ordered sequence of tool observations that have been produced by normal environment actions up to that turn. For each observation $O_j$, define $n_j = |O_j|$.
The theorem below concerns exact preservation of the observations $O_j$. It does not claim that discarded hidden reasoning traces are recoverable unless those traces are also stored verbatim.
\end{arcdefinition}

\begin{arcdefinition}[Usable prompt budget]
\label{def:usable_budget}
Let $L_{\mathrm{ctx}}$ be the model's total context limit and let $G_{\max}$ be the maximum number of completion tokens reserved for one model call. Define the usable input-token budget by
\begin{equation}
L = L_{\mathrm{ctx}} - G_{\max}.
\end{equation}
Thus, any model invocation whose serialized input contains at most $L$ tokens remains within the total context limit whenever the generated completion contains at most $G_{\max}$ tokens.
\end{arcdefinition}

\begin{arcassumption}[Injective fixed-width occurrence identifiers]
\label{ass:id}
Fix a declared maximum number of stored observation occurrences $N_{\max} \geq 2$, and define
\begin{equation}
b = \left\lceil \log_2\left(N_{\max}+1\right) \right\rceil.
\end{equation}
For each observation occurrence $j \in \{1,\ldots,N_{\max}\}$, let $\operatorname{enc}_b(j)$ be its injective $b$-bit encoding. Let $\operatorname{fp}_d$ be any deterministic $d$-bit fingerprint function. ARC assigns
\begin{equation}
\text{id}_j = \operatorname{enc}_b(j) \mathbin{\parallel} \operatorname{fp}_d\!\left( \operatorname{signature}(A_j) \mathbin{\parallel} \texttt{0x1F} \mathbin{\parallel} O_j \right).
\end{equation}
The occurrence index, rather than the fingerprint, provides uniqueness. A SHA1 value may be retained as the fingerprint, but none of the proofs assumes collision resistance of SHA1 or of a truncated digest.
\end{arcassumption}

\begin{arcassumption}[Append-only observation store]
\label{ass:store}
Immediately after a normal action $A_j$ produces observation $O_j$, and before any compaction can remove $O_j$ from the active transcript, ARC writes a record under $\text{id}_j$ satisfying
\begin{equation}
\texttt{ObsStore}[\text{id}_j].\texttt{full\_content} = O_j.
\end{equation}
After this write, the record is never overwritten, truncated, or deleted during the task.
\end{arcassumption}

\begin{arcassumption}[Persistent external catalog with bounded pages]
\label{ass:catalog}
For every stored observation occurrence $j$, ARC appends a citation stub $z_j$ to an external catalog. The stub contains $\text{id}_j$, bounded metadata, bounded previews, and a recall hint. Its serialized token charge, including delimiters, is at most a fixed constant $c \geq 1$. Fix a page capacity $p \geq 1$. At turn $i$, page $s$ is
\begin{equation}
\mathcal{G}_{i,s} = z_{(s-1)p+1} \mathbin{\parallel} \cdots \mathbin{\parallel} z_{\min\{sp,N_i\}},
\end{equation}
for $1 \leq s \leq \left\lceil \nicefrac{N_i}{p} \right\rceil$.
Only one page is materialized in the active prompt at a time. Define the catalog-page token budget $P$ so that $pc \leq P$.
Catalog navigation is stateful, so commands such as \texttt{catalog-first} and \texttt{catalog-next} have bounded prompt cost.
\end{arcassumption}

\begin{arcassumption}[Budgeted active view]
\label{ass:view}
Let $F = [\text{system},g]$
be the fixed task prefix, and define
$B = \operatorname{tokens}(F)$.
Before every call to $\operatorname{Model}$, ARC constructs the active transcript
\begin{equation}
\text{transcript}'_i = \operatorname{Serialize}\!\left( F,S_i,H_i,G_i,U_i \right),
\end{equation}
where $S_i$ is the deterministic summary, $H_i$ is the recent suffix, $G_i$ is one catalog page, and $U_i$ is the materialized recall buffer. Their budgets satisfy
\begin{equation}
\operatorname{tokens}(S_i) \leq M,
\end{equation}
\begin{equation}
\operatorname{tokens}(H_i) \leq R,
\end{equation}
\begin{equation}
\operatorname{tokens}(G_i) \leq P,
\end{equation}
\begin{equation}
\operatorname{tokens}(U_i) \leq Q.
\end{equation}
The recent suffix is selected by a token budget $R$, not by retaining a fixed number of turn pairs. Let $\eta_i$ be the exact serialization overhead for role markers, separators, field names, and control instructions at turn $i$. Assume
$0 \leq \eta_i \leq \eta$
for a fixed constant $\eta$, and assume that the serializer is charged so that
\begin{multline}
\operatorname{tokens}(\text{transcript}'_i) \leq B + \operatorname{tokens}(S_i) + \operatorname{tokens}(H_i) + \\
\operatorname{tokens}(G_i) + \operatorname{tokens}(U_i) + \eta_i.
\end{multline}

Define
\begin{equation}
K = B + M + R + P + Q + \eta,
\end{equation}
and require
$K \leq L$.
If a newly produced observation is too large to fit in $H_i$, ARC stores it first and represents it by its bounded citation stub before the next model invocation.
\end{arcassumption}

\begin{arcassumption}[Exact chunked recall]
\label{ass:recall}
Fix a positive recall payload budget $q \geq 1$ and a bounded response-wrapper cost $\rho \geq 0$ satisfying
$q + \rho \leq Q.$
For every nonempty stored observation $O_j$, define the number of recall chunks by $m_j = \lceil \nicefrac{n_j}{q} \rceil$.
For each chunk index $\ell \in \{1,\ldots,m_j\}$, define
\begin{equation}
a_{j,\ell} = (\ell-1)q+1,
\end{equation}
\begin{equation}
b_{j,\ell} = \min\{\ell q,n_j\},
\end{equation}
and define the returned payload by
\begin{equation}
\mathcal{R}_{j,\ell} = O_j[a_{j,\ell}:b_{j,\ell}].
\end{equation}
The commands \texttt{\_recall-start id} and \texttt{\_recall-next id} return these chunks in increasing order of $\ell$ using an external cursor. A recall reads only \texttt{ObsStore}; it never invokes $\operatorname{Execute}$. The recall buffer applies eviction if necessary so that its total serialized token cost remains at most $Q$. For an empty observation, define $m_j=0$, and let the empty concatenation recover the empty sequence.
\end{arcassumption}

\subsection{Why Simultaneously Visible Citation Persistence Cannot Be Bounded}

\begin{arcproposition}[Impossibility of all-visible citation persistence]
\label{prop:visible_citations_impossible}
Suppose that, after citing $n$ distinct observations, all $n$ citation stubs must remain simultaneously visible in the active prompt. If each visible citation incurs a token cost of at least $c_{\min}>0$, then no constant $K$ can bound the active prompt length for arbitrarily large $n$.
\end{arcproposition}

\begin{proof}
Because all $n$ citation stubs are simultaneously visible and each requires at least $c_{\min}$ tokens, the active prompt length $C(n)$ satisfies $C(n)\geq nc_{\min}$.

Suppose, for contradiction, that there exists a constant $K$ such that $C(n)\leq K$ for every positive integer $n$. Choose
\begin{equation}
n^{\star}=\left\lfloor\frac{K}{c_{\min}}\right\rfloor+1.
\end{equation}
By the defining property of the floor function,
\begin{equation}
\left\lfloor\frac{K}{c_{\min}}\right\rfloor
\leq \frac{K}{c_{\min}}
<
\left\lfloor\nicefrac{K}{c_{\min}}\right\rfloor+1
=n^{\star}.
\end{equation}
Thus, $\nicefrac{K}{c_{\min}}<n^{\star}$. Since $c_{\min}>0$, multiplying both sides by $c_{\min}$ yields $K<n^{\star}c_{\min}$. Consequently,
\begin{equation}
C(n^{\star})\geq n^{\star}c_{\min}>K,
\end{equation}
contradicting the assumed bound $C(n^{\star})\leq K$.

Therefore, a turn-independent bound on the active prompt is incompatible with keeping all citation stubs simultaneously visible. Persistent citation addressability must instead be maintained in external state, with only a bounded page or subset materialized in the active prompt at any given time.
\end{proof}

\subsection{Supporting Lemmas}

\begin{arclemma}[Identifier uniqueness]
\label{lem:id_unique}
Under Assumption~\ref{ass:id}, the identifiers $\text{id}*1,\ldots,\text{id}*{N_{\max}}$ are pairwise distinct.
\end{arclemma}

\begin{proof}
Fix any distinct indices $j,k\in{1,\ldots,N_{\max}}$. Suppose, for contradiction, that $\text{id}_j=\text{id}_k$. Each identifier is formed by concatenating a $b$-bit occurrence encoding with a $d$-bit fingerprint. Because both components have fixed width, equality of the complete identifiers implies equality of their first $b$ bits:
\begin{equation}
\operatorname{enc}_b(j)=\operatorname{enc}_b(k).
\end{equation}
By Assumption~\ref{ass:id}, $\operatorname{enc}*b$ is injective on ${1,\ldots,N*{\max}}$, so this equality implies $j=k$, contradicting the choice of distinct indices. Therefore, $\text{id}_j\neq\text{id}*k$. Since $j$ and $k$ were arbitrary, the identifiers $\text{id}*1,\ldots,\text{id}*{N*{\max}}$ are pairwise distinct.
\end{proof}

\begin{arclemma}[Uniform active-view bound]
\label{lem:active_bound}
Under Assumption~\ref{ass:view}, every model invocation satisfies
\begin{equation}
\operatorname{tokens}(\text{transcript}'_i) \leq K \leq L.
\end{equation}
The bound does not depend on $i$ or on $N_i$.
\end{arclemma}

\begin{proof}
By the serializer bound in Assumption~\ref{ass:view},
\begin{multline}
\operatorname{tokens}(\text{transcript}'_i) \leq B + \operatorname{tokens}(S_i) + \operatorname{tokens}(H_i) + \\
\operatorname{tokens}(G_i) + \operatorname{tokens}(U_i) + \eta_i.
\end{multline}
Applying the component-wise budget inequalities gives
\begin{equation}
\operatorname{tokens}(\text{transcript}'_i) \leq B+M+R+P+Q+\eta.
\end{equation}
By the definition of $K$,
\begin{equation}
B+M+R+P+Q+\eta = K.
\end{equation}
Therefore,
\begin{equation}
\operatorname{tokens}(\text{transcript}'_i) \leq K.
\end{equation}
The required budget condition gives
$K \leq L$.
Combining the last two inequalities yields
\begin{equation}
\operatorname{tokens}(\text{transcript}'_i) \leq K \leq L.
\end{equation}
The quantities $B,M,R,P,Q,$ and $\eta$ are fixed for the task and do not contain $i$ or $N_i$. Hence $K$ is independent of both $i$ and $N_i$.
\end{proof}

\begin{arclemma}[Bounded-page persistence and finite discoverability]
\label{lem:catalog}
Under Assumption~\ref{ass:catalog}, every citation stub $z_j$, for $1\leq j\leq N_i$, appears on exactly one catalog page, and each page has token cost at most $P$. Starting from the first page and repeatedly issuing \texttt{catalog-next}, the page containing $z_j$ is reached within $\left\lceil\nicefrac{j}{p}\right\rceil$ page materializations and, consequently, within $\left\lceil\nicefrac{N_i}{p}\right\rceil$ page materializations.
\end{arclemma}

\begin{proof}
Fix $j\in{1,\ldots,N_i}$, and define $s_j=\left\lceil\nicefrac{j}{p}\right\rceil$. By the defining property of the ceiling function,
\begin{equation}
s_j-1<\frac{j}{p}\leq s_j.
\end{equation}
Since $p>0$, multiplying throughout by $p$ yields
\begin{equation}
(s_j-1)p<j\leq s_jp.
\end{equation}
Because $j$ is an integer, the strict lower bound implies $(s_j-1)p+1\leq j$. Hence,
\begin{equation}
(s_j-1)p+1\leq j\leq s_jp.
\end{equation}
By the page definition in Assumption~\ref{ass:catalog}, $z_j$ therefore appears on page $s_j$.

To establish uniqueness, suppose, for contradiction, that $z_j$ also appears on a distinct page $r$. Without loss of generality, let $r>s_j$. Membership in page $s_j$ implies $j\leq s_jp$, whereas membership in page $r$ implies $j\geq(r-1)p+1$. Since $r\geq s_j+1$, we have $(r-1)p+1\geq s_jp+1$, and therefore $j\geq s_jp+1$. This contradicts $j\leq s_jp$. Thus, $z_j$ appears on exactly one catalog page.

Each page contains at most $p$ citation stubs, and each stub incurs a token cost of at most $c$. The token cost of a page is therefore at most $pc$. By Assumption~\ref{ass:catalog}, $pc\leq P$, so every materialized catalog page has token cost at most $P$.

Finally, a sequential scan beginning at the first page reaches page $s_j$ after $s_j=\left\lceil\nicefrac{j}{p}\right\rceil$ page materializations. Since $j\leq N_i$ and the ceiling function is monotone,
\begin{equation}
\left\lceil\frac{j}{p}\right\rceil
\leq
\left\lceil\frac{N_i}{p}\right\rceil.
\end{equation}
Therefore, the page containing $z_j$ is discoverable within $\left\lceil\nicefrac{N_i}{p}\right\rceil$ page materializations.
\end{proof}

\begin{arclemma}[Exact non-overlapping chunk reconstruction]
\label{lem:chunk_reconstruction}
Under Assumption~\ref{ass:recall}, for every nonempty observation $O_j$, the chunks $\mathcal{R}_{j,1},\ldots,\mathcal{R}_{j,m_j}$ form a disjoint, order-preserving partition of $O_j$. In particular,
\begin{equation}
\mathcal{R}_{j,1} \mathbin{\parallel} \mathcal{R}_{j,2} \mathbin{\parallel} \cdots \mathbin{\parallel} \mathcal{R}_{j,m_j} = O_j.
\end{equation}
Moreover, every chunk has at most $q$ payload tokens, and the number of chunks is exactly $m_j=\lceil n_j/q\rceil$.
\end{arclemma}

\begin{proof}
Fix a nonempty observation $O_j$, and write $n=n_j$ and $m=m_j=\left\lceil\nicefrac{n}{q}\right\rceil$. Since $n\geq 1$ and $q\geq 1$, we have $m\geq 1$. By the defining property of the ceiling function,
\begin{equation}
m-1<\frac{n}{q}\leq m,
\end{equation}
and hence
\begin{equation}
(m-1)q<n\leq mq.
\end{equation}

We first show that the chunks cover every token position. Fix any $t\in{1,\ldots,n}$ and define $\ell(t)=\left\lceil\nicefrac{t}{q}\right\rceil$. Because $1\leq t\leq n$ and the ceiling function is monotone,
\begin{equation}
1\leq \ell(t)\leq \left\lceil\frac{n}{q}\right\rceil=m.
\end{equation}
Moreover, the definition of $\ell(t)$ gives
\begin{equation}
(\ell(t)-1)q<t\leq \ell(t)q.
\end{equation}
Since $t$ is an integer, the strict left inequality implies $(\ell(t)-1)q+1\leq t$. Therefore,
\begin{equation}
a_{j,\ell(t)}=(\ell(t)-1)q+1\leq t.
\end{equation}
Because $t\leq \ell(t)q$ and $t\leq n$, we also have
\begin{equation}
t\leq \min{\ell(t)q,n}=b_{j,\ell(t)}.
\end{equation}
Thus $a_{j,\ell(t)}\leq t\leq b_{j,\ell(t)}$, so every token position belongs to at least one chunk.

We next establish that the chunks are pairwise disjoint. Suppose, for contradiction, that some token position $t$ belongs to two chunks indexed by $r<s$. Membership in chunk $r$ implies
\begin{equation}
t\leq b_{j,r}=\min{rq,n}\leq rq.
\end{equation}
Membership in chunk $s$ implies
\begin{equation}
t\geq a_{j,s}=(s-1)q+1.
\end{equation}
Since $s\geq r+1$, it follows that $(s-1)q+1\geq rq+1$, and therefore $t\geq rq+1$. This contradicts $t\leq rq$. Hence, no token position belongs to more than one chunk.

It remains to verify the chunk lengths. For every $\ell\in{1,\ldots,m-1}$, the inequality $(m-1)q<n$ implies $\ell q<n$, and hence $b_{j,\ell}=\ell q$. Therefore,
\begin{equation}
|\mathcal{R}*{j,\ell}|
=b*{j,\ell}-a_{j,\ell}+1
=\ell q-\bigl((\ell-1)q+1\bigr)+1
=q.
\end{equation}
For the final chunk, $n\leq mq$ gives $b_{j,m}=n$, so
\begin{equation}
|\mathcal{R}*{j,m}|
=n-\bigl((m-1)q+1\bigr)+1
=n-(m-1)q.
\end{equation}
The inequalities $(m-1)q<n\leq mq$ further imply
\begin{equation}
1\leq |\mathcal{R}*{j,m}|\leq q.
\end{equation}
Thus every chunk contains at most $q$ payload tokens, and the total payload length is
\begin{equation}
\begin{aligned}
\sum_{\ell=1}^{m}|\mathcal{R}*{j,\ell}|
&=\sum*{\ell=1}^{m-1}q+\bigl(n-(m-1)q\bigr)\\
&=(m-1)q+n-(m-1)q\\
&=n.
\end{aligned}
\end{equation}

The chunk intervals are disjoint, exhaustive, and ordered by $\ell$. Their concatenation therefore recovers all $n$ tokens of $O_j$ in their original order:
\begin{equation}
\mathcal{R}*{j,1}\mathbin{\parallel}\cdots\mathbin{\parallel}\mathcal{R}*{j,m}=O_j.
\end{equation}
Finally, the number of chunks is $m=\left\lceil\nicefrac{n}{q}\right\rceil$ by definition, completing the proof.
\end{proof}

\begin{arccorollary}[Recall-action and serialization overhead]
\label{cor:recall_overhead}
For every nonempty observation $O_j$, the number of exact recall responses satisfies
\begin{equation}
\frac{n_j}{q} \leq m_j < \frac{n_j}{q}+1.
\end{equation}
The total recalled payload is exactly $n_j$ tokens. If each response has wrapper cost at most $\rho$, the total serialized size of all recall responses is at most
\begin{equation}
n_j+\rho\left(\frac{n_j}{q}+1\right).
\end{equation}
\end{arccorollary}

\begin{proof}
By definition, $m_j=\left\lceil\nicefrac{n_j}{q}\right\rceil$. For any real number $x$, the ceiling function satisfies $x\leq\lceil x\rceil<x+1$. Substituting $x=\nicefrac{n_j}{q}$ therefore gives $\nicefrac{n_j}{q}\leq m_j<\nicefrac{n_j}{q}+1$. By Lemma~\ref{lem:chunk_reconstruction}, the recalled responses collectively contain exactly the original payload:
\begin{equation}
\sum_{\ell=1}^{m_j}\left|\mathcal{R}_{j,\ell}\right|=n_j.
\end{equation}
Because each of the $m_j$ responses contributes at most $\rho$ wrapper tokens, their total serialized size is at most $n_j+\rho m_j$. Moreover, since $\rho\geq 0$ and $m_j<\nicefrac{n_j}{q}+1$, multiplying by $\rho$ yields
\begin{equation}
\rho m_j\leq \rho\left(\nicefrac{n_j}{q}+1\right).
\end{equation}
Consequently,
\begin{equation}
n_j+\rho m_j
\leq
n_j+\rho\left(\nicefrac{n_j}{q}+1\right),
\end{equation}
which establishes the claimed upper bound.
\end{proof}

\subsection{Main Guarantee}

\begin{arctheorem}[Bounded active view, observational losslessness, and exact address-conditioned recall]
\label{thm:arc_formal}
Consider any ARC execution satisfying Assumptions~\ref{ass:id}--\ref{ass:recall}, and suppose $N_i\leq N_{\max}$ at turn $i$. Then all of the following statements hold.

\begin{enumerate}
    \item \textbf{Prompt safety.} Every model invocation uses an active transcript satisfying
    \begin{equation}
\operatorname{tokens}(\text{transcript}'_i) \leq K \leq L.
\end{equation}

    \item \textbf{Stable archival fidelity.} For every $j\in\{1,\ldots,N_i\}$ and every later turn $t\geq i$ before the declared capacity is exhausted, the identifier $\text{id}_j$ remains unique and
    \begin{equation}
\texttt{ObsStore}_t[\text{id}_j].\texttt{full\_content}=O_j.
\end{equation}

    \item \textbf{Persistent finite discoverability.} The citation stub for every $O_j$ remains in the external catalog and can be materialized by a finite page scan. At most one page of token cost $P$ is visible at a time.

    \item \textbf{Exact address-conditioned recall without re-execution.} Given a valid $\text{id}_j$, ARC returns $O_j$ exactly in
    \begin{equation}
m_j=\left\lceil\frac{n_j}{q}\right\rceil
\end{equation}
    non-overlapping recall responses, each containing at most $q$ payload tokens, and no recall response invokes $\operatorname{Execute}(A_j)$.

    \item \textbf{Observational losslessness.} The ordered observation history $\mathcal{O}_i$ can be reconstructed exactly from \texttt{ObsStore} and the stored creation indices. Consequently, the map from ordered observation histories to the archived system state is injective.

    \item \textbf{Separation of active and archival growth.} The active-view bound $K$ is independent of $i$ and $N_i$, whereas the exact archived payload equals $\sum_{j=1}^{N_i} n_j$
    tokens, excluding metadata.
\end{enumerate}
\end{arctheorem}

\begin{proof}
We prove the six claims in order.

\textbf{1. Prompt safety.}
Lemma~\ref{lem:active_bound} gives
\begin{equation}
\operatorname{tokens}(\text{transcript}'_i) \leq K \leq L.
\end{equation}
This holds before every model invocation by Assumption~\ref{ass:view}. Because $L=L_{\mathrm{ctx}}-G_{\max}$, adding any completion of at most $G_{\max}$ tokens gives
\begin{equation}
\operatorname{tokens}(\text{transcript}'_i)+G_{\max} \leq L+G_{\max}.
\end{equation}
Substituting the definition of $L$ gives
\begin{equation}
L+G_{\max} = \left(L_{\mathrm{ctx}}-G_{\max}\right)+G_{\max}.
\end{equation}
Cancelling $-G_{\max}$ and $+G_{\max}$ yields
\begin{equation}
L+G_{\max}=L_{\mathrm{ctx}}.
\end{equation}
Therefore,
\begin{equation}
\operatorname{tokens}(\text{transcript}'_i)+G_{\max} \leq L_{\mathrm{ctx}}.
\end{equation}
Thus the prompt and reserved completion fit within the total context limit.

\textbf{2. Stable archival fidelity.}
Identifier uniqueness follows from Lemma~\ref{lem:id_unique}. We now prove persistence of the stored content by induction over host-system transitions after the insertion of $O_j$.

Immediately after insertion, Assumption~\ref{ass:store} gives
\begin{equation}
\texttt{ObsStore}[\text{id}_j].\texttt{full\_content}=O_j.
\end{equation}
This is the induction base case. Assume that after some later host-system transition $u$, the equality still holds:
\begin{equation}
\texttt{ObsStore}_u[\text{id}_j].\texttt{full\_content}=O_j.
\end{equation}
Consider transition $u+1$. There are three relevant cases.

First, a new normal action may insert a new observation $O_k$. Because it is a new occurrence, $k\neq j$. Lemma~\ref{lem:id_unique} gives
$\text{id}_k\neq \text{id}_j$.
The insertion under $\text{id}_k$ therefore does not overwrite the record under $\text{id}_j$.

Second, a recall action reads an existing record. By Assumption~\ref{ass:recall}, recall performs no write to \texttt{ObsStore}. Hence the record under $\text{id}_j$ is unchanged.

Third, a compaction or catalog-navigation action modifies only the materialized active view, summary, recent suffix, page selection, or recall buffer. By Assumption~\ref{ass:store}, none of these operations overwrites, truncates, or deletes the record under $\text{id}_j$.

In every case,
\begin{equation}
\texttt{ObsStore}_{u+1}[\text{id}_j].\texttt{full\_content}=O_j.
\end{equation}
The induction step is complete. Therefore the equality holds at every later turn $t$.

\textbf{3. Persistent finite discoverability.}
Assumption~\ref{ass:catalog} appends $z_j$ to an external catalog and never removes it. Lemma~\ref{lem:catalog} shows that $z_j$ belongs to exactly one page, that the page costs at most $P$ tokens, and that a sequential page scan reaches it after at most
\begin{equation}
\left\lceil \frac{j}{p} \right\rceil \leq \left\lceil \frac{N_i}{p} \right\rceil
\end{equation}
page materializations. Thus the citation remains persistently discoverable without requiring all citations to be visible simultaneously.

\textbf{4. Exact address-conditioned recall without re-execution.}
Fix any $j\in\{1,\ldots,N_i\}$. By Claim 2, the store entry remains exactly equal to $O_j$. Lemma~\ref{lem:chunk_reconstruction} gives
\begin{equation}
\mathcal{R}_{j,1} \mathbin{\parallel} \cdots \mathbin{\parallel} \mathcal{R}_{j,m_j} = O_j,
\end{equation}
and each chunk has at most $q$ payload tokens. Assumption~\ref{ass:recall} states that the recall handler reads these chunks from \texttt{ObsStore} and never invokes $\operatorname{Execute}$. Therefore the original environment action is not re-executed, and changes in the external environment after the original action cannot alter the archived return value.

\textbf{5. Observational losslessness.}
Define a decoder $D_i$ that sorts the archive records by their stored creation indices and outputs their \texttt{full\_content} fields:
\begin{multline}
    D_i(\texttt{ObsStore}_i) = \\
    \Bigl( \texttt{ObsStore}_i[\text{id}_1].\texttt{full\_content}, \ldots, \\ \texttt{ObsStore}_i[\text{id}_{N_i}].\texttt{full\_content} \Bigr).
\end{multline}
By Claim 2, each component satisfies
\begin{equation}
\texttt{ObsStore}_i[\text{id}_j].\texttt{full\_content}=O_j.
\end{equation}
Substituting these equalities component by component gives
\begin{equation}
D_i(\texttt{ObsStore}_i) = (O_1,O_2,\ldots,O_{N_i}).
\end{equation}
By Definition~\ref{def:observation_history},
\begin{equation}
(O_1,O_2,\ldots,O_{N_i})=\mathcal{O}_i.
\end{equation}
Hence
\begin{equation}
D_i(\texttt{ObsStore}_i)=\mathcal{O}_i.
\end{equation}
To prove injectivity, suppose two ordered observation histories $\mathcal{O}_i$ and $\widetilde{\mathcal{O}}_i$ produce the same archived system state. Applying the same deterministic decoder to that state gives
\begin{equation}
D_i(\texttt{ObsStore}_i)=\mathcal{O}_i
\end{equation}
and
\begin{equation}
D_i(\texttt{ObsStore}_i)=\widetilde{\mathcal{O}}_i.
\end{equation}
By transitivity of equality,
\begin{equation}
\mathcal{O}_i=\widetilde{\mathcal{O}}_i.
\end{equation}
Therefore distinct ordered observation histories cannot map to the same archived state, so the map is injective.

\textbf{6. Separation of active and archival growth.}
By Claim 1, the active prompt is bounded by
\begin{equation}
K=B+M+R+P+Q+\eta,
\end{equation}
where every term on the right-hand side is fixed independently of $i$ and $N_i$. Therefore the active-view size is uniformly bounded.

By Assumption~\ref{ass:store}, the archive stores one verbatim payload $O_j$ for each occurrence $j$. The payload token count is therefore
\begin{equation}
\begin{aligned}
\operatorname{PayloadTokens}(\texttt{ObsStore}_i) &= |O_1|+|O_2|+\cdots+|O_{N_i}| \\
&= n_1+n_2+\cdots+n_{N_i} \\
&= \sum_{j=1}^{N_i}n_j.
\end{aligned}
\end{equation}
This quantity may grow with the trajectory, but it grows only in external archival state, not in the active prompt. This proves the claimed separation.
\end{proof}

\begin{arcremark}[Scope of the guarantee]
Theorem~\ref{thm:arc_formal} is a mechanism-level guarantee. It proves prompt safety, exact archival preservation, finite discoverability, and exact responses to valid recall commands. It does not prove that the language model will select the correct identifier, issue enough recall commands, or solve the downstream task correctly. It is also an \emph{observation-losslessness} theorem. If the paper wishes to claim full-transcript losslessness, then every discarded action and every discarded reasoning record must also be archived under the same append-only discipline.
\end{arcremark}

\subsection{Worst-Case External-Memory Lower Bound}

\begin{arcproposition}[Exact recovery requires linearly growing external memory]
\label{prop:storage_lower_bound}
Consider any deterministic memory system that receives an arbitrary $N$-bit history $x\in\{0,1\}^N$. Suppose the active representation has at most $2^{\kappa}$ distinguishable states, the external memory has at most $2^s$ distinguishable states, and a deterministic decoder must recover every $x$ exactly from the pair of states. Then
\begin{equation}
s \geq N-\kappa.
\end{equation}
Consequently, for fixed $\kappa$, worst-case exact recovery requires external storage in $\Omega(N)$ bits. A prompt backed by $\kappa$ fixed binary storage cells is a special case because it has exactly $2^{\kappa}$ possible cell configurations.
\end{arcproposition}

\begin{proof}
There are exactly $2^N$ possible histories in ${0,1}^N$. By assumption, the active representation and external memory can occupy at most $2^\kappa$ and $2^s$ distinguishable states, respectively. Hence, the total number of distinguishable active--external state pairs is at most $2^\kappa 2^s = 2^{\kappa+s}$. Exact recovery of every history requires the encoding from histories to state pairs to be injective: if two distinct histories were mapped to the same state pair, a deterministic decoder would receive identical inputs for both and therefore could not recover both correctly. Consequently, $2^N \leq 2^{\kappa+s}$. Since $2^x$ is strictly increasing, taking base-two logarithms yields $N \leq \kappa+s$, or equivalently, $s \geq N-\kappa$. For fixed $\kappa$ and any $N \geq 2\kappa$, we have $\kappa \leq \nicefrac{N}{2}$, and thus $N-\kappa \geq \nicefrac{N}{2}$. Therefore, $s \geq N-\kappa \geq N/2$ for all $N \geq 2\kappa$, establishing that $s \in \Omega(N)$.

\end{proof}

\begin{arccorollary}[First-order payload optimality of append-only ARC]
\label{cor:payload_optimality}
For an incompressible $N$-bit observation history and a fixed active-state budget $\kappa$, Proposition~\ref{prop:storage_lower_bound} establishes a worst-case lower bound of $N-\kappa$ bits of external storage. An append-only ARC archive stores exactly $N$ observation-payload bits, excluding metadata, and therefore exceeds this lower bound by only $N-(N-\kappa)=\kappa$ bits. Moreover, for $N>\kappa$, $(\nicefrac{N}{N-\kappa})=(\nicefrac{1}{1-\kappa/N})$, which implies that $\lim_{N\to\infty}(\nicefrac{N}{N-\kappa})=1$. Hence, among methods that support exact recovery of arbitrary histories while maintaining a bounded active state, ARC's verbatim payload storage is asymptotically optimal to first order.
\end{arccorollary}

\begin{proof}
The lower bound of $N-\kappa$ bits follows directly from Proposition~\ref{prop:storage_lower_bound}. Because ARC stores all $N$ input bits verbatim, its payload storage is exactly $N$ bits. Therefore, its additive overhead relative to the lower bound is $N-(N-\kappa)=\kappa$ bits.

To analyze the multiplicative overhead, divide both the numerator and denominator by $N>0$:
\begin{equation}
\frac{N}{N-\kappa}=\frac{1}{1-\kappa/N}.
\end{equation}
For fixed $\kappa$,
\begin{equation}
\lim_{N\to\infty}\frac{\kappa}{N}=0,
\end{equation}
and hence
\begin{equation}
\lim_{N\to\infty}\left(1-\frac{\kappa}{N}\right)=1.
\end{equation}
By continuity of the reciprocal function at $1$, it follows that
\begin{equation}
\lim_{N\to\infty}\frac{1}{1-\kappa/N}=1.
\end{equation}
Thus, the ratio between ARC's payload storage and the information-theoretic lower bound converges to $1$, establishing ARC's first-order asymptotic optimality.
\end{proof}

\begin{arcremark}[What the lower bound does and does not say]
The lower bound is a worst-case result that applies to the exact recovery of arbitrary histories. Compressible histories may require less external storage, and the bound excludes implementation-specific metadata. Its purpose is to establish that a bounded active prompt and exact recoverability cannot both be achieved with constant total memory: any information removed from the prompt must be represented elsewhere in the external memory.
\end{arcremark}

%% file: ablation.tex
\section{Ablation}
\label{sec:ablation}

ARC's behavior is governed by five parameters, each of which trades off how often the compaction fires against how much of \texttt{ObsStore} is allowed back into the visible transcript at once. We describe them here as part of the method because each is varied independently in our Ablation Section; together they bound the conditions under which ARC's guarantee is exercised versus dormant.

\begin{description}
\item[Context-window budget ($\mathit{max\_model\_len}$, $L$).] \textbf{Default=16,384}. The token limit the serving engine enforces on prompt + completion length (Section 2.1); this is exactly $L$ in Algorithm 1. Smaller values make the overflow condition $C_i > L$ true more often, so compaction --- and with it ARC's citation-and-recall mechanism --- is exercised more frequently; larger values let the same trajectory run further before fires, and at the limit ($L \to \infty$) ARC, Auto-Compact, and full-context converge because none of them are ever invoked.
\item[Recall working-set budget ($\mathit{recall\_budget\_chars}$).] \textbf{Default=16,000}. The total character budget for observations currently \emph{materialized} --- i.e., fetched back via \texttt{\_recall} and expanded verbatim in the transcript --- at once. This is a constraint on the recall path rather than on \texttt{ObsStore} itself: \texttt{ObsStore} is never bounded, but the \emph{live, expanded} content drawn from it is capped so that recall cannot itself re-create the unbounded growth ARC is designed to avoid. When a new \texttt{\_recall} \S\texttt{id} would push the materialized total over this budget, the least-recently-used materialized observation is evicted back to its citation stub (its \S\texttt{id} remains visible; only the expansion is reverted).
\item[Per-recall cap ($\mathit{recall\_max\_chars}$).] \textbf{Default=8,000}. The maximum size of a single \texttt{\_recall} \S\texttt{id} response. If the stored observation is larger than this, the response is windowed to a head/tail excerpt rather than injected in full --- the same head+tail structure used in the citation stub, just with a larger window. This bounds the worst case where one oversized observation (e.g., a full-repository \texttt{grep} dump) would otherwise consume the entire recall working-set budget in a single call.
\item[Forced-compaction schedule ($\mathit{force\_at\_calls}$).] \textbf{Default=$(5, 10, 15)$}. A set of call indices at which compaction is triggered experimentally regardless of whether $C_i$ has organically exceeded $L$. This lets the compaction \emph{frequency} be controlled directly, isolating its effect from the task- and model-specific rate at which context happens to grow, and is used only in the evaluation, not in deployment, where compaction fires solely on a overflow.
\item[Fixture size ($\mathit{noise\_lines\_per\_side}$).] \textbf{Default=25 as a benchmark parameter}. Not a parameter of ARC itself but of the needle-in-a-haystack harness: the number of distractor lines placed on each side of the needle banner inside input file, which determines the size of the observation that compaction must decide whether to keep, drop, or cite (at the default, $\approx$ 4.5k characters --- above the 500-character keep-verbatim threshold used by all lossy baselines, but below the 8,000-character limit at which the \emph{presentation} layer would elide the needle before compaction ever runs). Increasing it raises the cost of losing the needle under lossy compaction while leaving ARC's citation size unaffected, since a citation's cost is bounded by the head/tail preview rather than by the underlying observation's length. 
\end{description}

\subsection{Recall\_budget for Longbench task}
\textbf{Setup.}
We ablate the recall budget by jointly sweeping \texttt{RECALL\_BUDGET\_CHARS} and \texttt{RECALL\_MAX\_CHARS} (Table~\ref{tab:ablation_recall_longbench}) at three levels---8k/4k, 16k/8k, and 32k/16k---across two model scales (Qwen3-8B, 16k context; Qwen3-32B, 32k context) and five memory/compaction methods (\textsc{Full\_context, Sliding\_window, LLM\_summary, Structured\_state, RAG\_memory}) compared against our method, ARC. Accuracy is computed as
\[
\mathrm{Accuracy} = \frac{\mathrm{Success}}{\mathrm{Success} + \mathrm{Fail}}
\]
over the LongBench evaluation set ($n = 311$ hard tasks).

\textbf{Full\_context collapses at every setting.}
\textsc{Full\_context} scores $0.00\%$ accuracy across all budgets and both scales, confirming that LongBench inputs exceed the base context window regardless of the recall allowance and establishing that some form of compaction is required rather than optional.

\textbf{ARC's accuracy is a single-peaked, inverted-U function of recall budget, not monotonic.}
On Qwen3-8B, ARC accuracy moves from $26.05\% \rightarrow 29.26\% \rightarrow 27.33\%$ as the budget grows from 8k/4k to 16k/8k to 32k/16k; on Qwen3-32B it moves from $29.90\% \rightarrow 34.08\% \rightarrow 33.76\%$. In both cases the peak occurs at the intermediate 16k/8k budget rather than the largest. This indicates ARC's recall mechanism benefits from a moderate, well-targeted budget and degrades once the budget grows large enough to admit lower-relevance content, providing evidence that ARC's gains come from \emph{recall quality} rather than simply retrieving more.

\textbf{ARC is the best or near-best method at practical (16k/8k and 32k/16k) budgets on both scales.}

\begin{table}[h]
\centering
\setlength{\tabcolsep}{2.5pt}
\begin{tabular}{llccc}
\toprule
\rowcolor{gray!12}
\textbf{Scale} & \textbf{Budget} & \textbf{ARC} & \textbf{Runner-up} & \textbf{Margin} \\
\midrule
\midrule
8B  & 16/8k   & 29.26\% & Structured\_state 26.05\% & +3.2pp \\
8B  & 32/16k  & 27.33\% & RAG\_memory 24.76\%        & +2.6pp \\
32B & 16/8k   & 34.08\% & Sliding\_window 30.87\%    & +3.2pp \\
32B & 32/16k  & 33.76\% & Sliding\_window 29.58\%    & +4.2pp \\
\bottomrule
\end{tabular}
\end{table}

At the smallest budget (8k/4k), ARC is competitive but not the top performer: \textsc{Sliding\_window} leads on Qwen3-32B (31.19\%), while \textsc{LLM\_summary} and \textsc{Structured\_state} slightly outperform ARC on Qwen3-8B. This suggests ARC requires a minimum recall budget to construct a useful memory structure, which we note as a limitation under very tight budgets rather than omit.

\textbf{Baseline methods do not show the same peak-then-decline profile as ARC, and some degrade sharply.}
\textsc{LLM\_summary} is flat-to-declining as budget grows on Qwen3-32B ($24.76\% \rightarrow 22.51\% \rightarrow 24.12\%$) and peaks early before dropping on Qwen3-8B ($26.69\% \rightarrow 26.69\% \rightarrow 23.15\%$). Structured\_state exhibits the steepest degradation of any method on Qwen3-8B, falling from $27.01\%$ to $21.22\%$ ($-5.8$pp) as the budget increases from 8k/4k to 32k/16k, consistent with excess retrieved state introducing noise rather than signal. \textsc{Sliding\_window} and \textsc{RAG\_memory} are comparatively flat or monotonically decreasing with budget; they neither benefit from additional budget in the transient manner observed for ARC nor collapse as sharply as \textsc{Structured\_state}.

\textbf{ARC's advantage grows with model scale.}
The margin between ARC and the next-best method is larger on Qwen3-32B than on Qwen3-8B at matched budgets (e.g., +4.2pp vs.\ +2.6pp at 32k/16k), suggesting ARC's recall mechanism scales favorably with model capacity, likely because larger models make better use of the structured, compacted recall content provided by ARC.

\subsection{Recall\_budget for Needle-in-a-haystack task}

For saving the computational cost, we only ablate recall\_budget for the Needle-in-a-haystack on small model Qwen3-8B in Table~\ref{tab:ablation_recall_budget_needle}.

\paragraph{Setup.}

We ablate the recall budget by jointly sweeping `RECALL\_BUDGET\_CHARS` and `RECALL\_MAX\_CHARS` at four levels — 4k/2k, 8k/4k, 16k/8k, and 32k/16k — on a synthetic needle-in-a-haystack retrieval task (n = 1000), comparing \textsc{Full\_context, Sliding\_window, LLM\_summary, Structured\_state}, and \textsc{RAG\_memory} against ARC. Accuracy is Success / (Success + Fail).

\textbf{ARC scales monotonically with recall budget and saturates near ceiling, unlike on LongBench.} 
ARC accuracy rises 92.5\% $\rightarrow$ 93.2\% $\rightarrow$ 99.1\% $\rightarrow$ 99.4\% as budget grows from 4k/2k to 32k/16k, with the large jump occurring at 16k/8k and only +0.3pp gained from doubling the budget again at 32k/16k. This is the opposite budget-sensitivity profile from the LongBench ablation (where ARC peaked at an intermediate budget and declined thereafter): on this synthetic, single-needle retrieval task there is little irrelevant content to be diluted by a larger budget, so additional recall capacity is purely additive until the task is essentially solved ($\sim$99\% accuracy, near the practical ceiling given task noise).

\textbf{ARC's margin over the next-best baseline widens as budget increases, then holds.} 
At the smallest budget (4k/2k) ARC leads the next-best method (\textsc{RAG\_memory}, 80.0\%) by +12.5pp; at 8k/4k it leads \textsc{LLM\_summary} (78.4\%) by +14.8pp; at 16k/8k and 32k/16k it leads \textsc{RAG\_memory} (78.8\% and 79.3\%) by +20.3pp and +20.1pp respectively. The margin roughly doubles once budget exceeds 8k/4k, indicating ARC converts additional recall budget into accuracy far more effectively than any baseline.

\textbf{Baseline methods plateau well below ARC's ceiling regardless of budget, showing their limitation is retrieval quality, not budget size.} 
Full\_context (35--40\%) and Sliding\_window (53--57\%) are essentially flat across the entire budget sweep --- a 4x increase in recall budget from 4k/2k to 16k/8k changes their accuracy by only a few points, meaning these methods are bottlenecked by their retrieval mechanism rather than by capacity. \textsc{LLM\_summary, Structured\_state, and RAG\_memory} occupy a middle band (69--80\%) that is likewise budget-insensitive: \textsc{RAG\_memory}, the strongest baseline, ranges only 77.5--80.0\% across the full sweep and never approaches ARC's 99\%+ accuracy at matched budgets.

\textbf{Contrast with the LongBench ablation strengthens the paper's core claim.} 
Together, the two ablations show that ARC's recall-budget behavior is task-dependent in a way baselines' is not: on LongBench (dense, multi-fact, real-world text) ARC shows an inverted-U with an optimum at moderate budget, while on needle-in-a-haystack (sparse, single-fact, synthetic text) ARC shows monotonic improvement to near-ceiling. Baselines, by contrast, are budget-insensitive on both tasks --- they neither peak nor saturate, they simply plateau below their respective ceilings. This supports the claim that ARC's improvements stem from \emph{how} it allocates and prioritizes recall content, not merely from having more of it.

\begin{table}[htbp]
\centering
\small
\setlength{\tabcolsep}{4pt}
\begin{tabular}{l|c|rrr}
\toprule
\rowcolor{gray!12}
\textbf{Method} & \textbf{Recall Budget} & \textbf{Success} & \textbf{Fail} & \textbf{Accuracy (\%)} \\
\midrule
\midrule
Full context     & \multirow{6}{*}{4k/2k}   & 375 & 625 & 37.5 \\
Sliding window   &                           & 570 & 430 & 57.0 \\
LLM summary      &                           & 762 & 238 & 76.2 \\
Structured state &                           & 730 & 270 & 73.0 \\
RAG memory       &                           & 800 & 200 & 80.0 \\
\textbf{ARC}     &                           & \textbf{925} & \textbf{75} & \textbf{92.5} \\
\midrule

Full context     & \multirow{6}{*}{8k/4k}   & 350 & 650 & 35.0 \\
Sliding window   &                           & 534 & 466 & 53.4 \\
LLM summary      &                           & 784 & 216 & 78.4 \\
Structured state &                           & 688 & 312 & 68.8 \\
RAG memory       &                           & 775 & 225 & 77.5 \\
\textbf{ARC}     &                           & \textbf{932} & \textbf{68} & \textbf{93.2} \\
\midrule

Full context     & \multirow{6}{*}{16k/8k}  & 396 & 604 & 39.6 \\
Sliding window   &                           & 571 & 429 & 57.1 \\
LLM summary      &                           & 733 & 267 & 73.3 \\
Structured state &                           & 688 & 312 & 68.8 \\
RAG memory       &                           & 788 & 212 & 78.8 \\
\textbf{ARC}     &                           & \textbf{991} & \textbf{9} & \textbf{99.1} \\
\midrule

Full context     & \multirow{6}{*}{32k/16k} & 373 & 627 & 37.3 \\
Sliding window   &                           & 543 & 457 & 54.3 \\
LLM summary      &                           & 759 & 241 & 75.9 \\
Structured state &                           & 699 & 301 & 69.9 \\
RAG memory       &                           & 793 & 207 & 79.3 \\
\textbf{ARC}     &                           & \textbf{994} & \textbf{6} & \textbf{99.4} \\
\bottomrule
\end{tabular}
\caption{Needle-in-a-Haystack evaluation on Qwen3-8B (16k context window, 4096 max-tokens) under different recall budgets. Each setting is evaluated on 1000 tasks.}
\label{tab:ablation_recall_budget_needle}
\end{table}

\begin{table*}[htbp]
\centering
\small
\setlength{\tabcolsep}{5pt}
\begin{tabular}{llcccccc}
\toprule
\rowcolor{gray!12}
\textbf{Model} & \textbf{Method} & \textbf{RECALL\_BUDGET\_CHARS} & \textbf{RECALL\_MAX\_CHARS} & \textbf{Success} & \textbf{Fail} & \textbf{Accuracy (\%)} \\
\midrule
\midrule
\rowcolor{blue!20}
\multicolumn{7}{l}{\textbf{Qwen3-8B (16k context)}} \\

& Full\_context     & 8k  & 4k  & 0   & 311 & 0.00 \\
& Sliding\_window   & 8k  & 4k  & 79  & 232 & 25.40 \\
& \textbf{LLM\_summary}      & 8k  & 4k  & \textbf{83}  & \textbf{228} & \textbf{26.69} \\
& Structured\_state & 8k  & 4k  & 84  & 227 & 27.01 \\
& RAG\_memory       & 8k  & 4k  & 80  & 231 & 25.72 \\
& ARC               & 8k  & 4k  & 81  & 230 & 26.05 \\
\cmidrule(lr){2-7}

& Full\_context     & 16k & 8k  & 0   & 311 & 0.00 \\
& Sliding\_window   & 16k & 8k  & 77  & 234 & 24.76 \\
& LLM\_summary      & 16k & 8k  & 83  & 228 & 26.69 \\
& Structured\_state & 16k & 8k  & 81  & 230 & 26.05 \\
& RAG\_memory       & 16k & 8k  & 76  & 235 & 24.44 \\
& \textbf{ARC}               & 16k & 8k  & \textbf{91}  & \textbf{220} & \textbf{29.26} \\
\cmidrule(lr){2-7}

& Full\_context     & 32k & 16k & 0   & 311 & 0.00 \\
& Sliding\_window   & 32k & 16k & 76  & 235 & 24.44 \\
& LLM\_summary      & 32k & 16k & 72  & 239 & 23.15 \\
& Structured\_state & 32k & 16k & 66  & 245 & 21.22 \\
& RAG\_memory       & 32k & 16k & 77  & 234 & 24.76 \\
& \textbf{ARC}               & 32k & 16k & \textbf{85}  & \textbf{226} & \textbf{27.33} \\
\midrule
\rowcolor{blue!20}
\multicolumn{7}{l}{\textbf{Qwen3-32B (32k context)}} \\

& Full\_context     & 8k  & 4k  & 0   & 311 & 0.00 \\
& Sliding\_window   & 8k  & 4k  & 97  & 214 & 31.19 \\
& LLM\_summary      & 8k  & 4k  & 77  & 234 & 24.76 \\
& \textbf{Structured\_state} & 8k  & 4k  & \textbf{91}  & \textbf{220} & \textbf{29.26} \\
& RAG\_memory       & 8k  & 4k  & 85  & 226 & 27.33 \\
& ARC               & 8k  & 4k  & 93  & 218 & 29.90 \\
\cmidrule(lr){2-7}

& Full\_context     & 16k & 8k  & 0   & 311 & 0.00 \\
& Sliding\_window   & 16k & 8k  & 96  & 215 & 30.87 \\
& LLM\_summary      & 16k & 8k  & 70  & 241 & 22.51 \\
& Structured\_state & 16k & 8k  & 93  & 218 & 29.90 \\
& RAG\_memory       & 16k & 8k  & 88  & 223 & 28.30 \\
& \textbf{ARC}               & 16k & 8k  & \textbf{106} & \textbf{205} & \textbf{34.08} \\
\cmidrule(lr){2-7}

& Full\_context     & 32k & 16k & 0   & 311 & 0.00 \\
& Sliding\_window   & 32k & 16k & 92  & 219 & 29.58 \\
& LLM\_summary      & 32k & 16k & 75  & 236 & 24.12 \\
& Structured\_state & 32k & 16k & 92  & 219 & 29.58 \\
& RAG\_memory       & 32k & 16k & 80  & 231 & 25.72 \\
& \textbf{ARC}               & 32k & 16k & \textbf{105} & \textbf{206} & \textbf{33.76} \\
\bottomrule
\end{tabular}
\caption{LongBench results across methods under different recall budgets. Accuracy is computed as Success / (Success + Fail).}
\label{tab:ablation_recall_longbench}
\end{table*}

\medskip
\subsection{HBM (High-Bandwidth Memory) Bandwidth Analysis}

\paragraph{1. Experimental setup.}
All HBM roofline results in this ablation section are obtained by replaying the recorded
LLM trajectories of each method on a single NVIDIA
\texttt{H200\_141GB\_SXM}. We instantiate the hardware model with
\textbf{141.0 GB} of HBM capacity, \textbf{4.8 TB/s} peak HBM bandwidth, and
\textbf{989.0 TFLOPS} dense BF16 Tensor Core throughput. Because the experiments
use one GPU only ($n_{\mathrm{gpu}} = 1$), the aggregate hardware terms in the
roofline model are exactly these per-device values.

The analysis scores each method on three quantities: total \textbf{decode HBM
traffic}, \textbf{prefill compute time}, and their sum, i.e., the
\textbf{roofline wall-clock estimate}. All methods are evaluated from the same
trajectory logs, so they share the same model weights and hardware constants;
differences arise from how each method changes the amount of live context that
must be retained and reread across turns.

\paragraph{2. Roofline model.}
For each LLM call $k$, let $p_k$ denote the prompt length and $m_k$ the number
of generated tokens. During autoregressive decode, each generated token streams
two components through HBM: (i) the model weights $M_w$, amortized across the
decode batch size $B$; and (ii) the request's KV cache, whose cost scales with
the active context length.

The per-token decode traffic is therefore
\begin{equation}
\mathrm{bytes/token} = M_w / B + L_k \times kv_b ,
\end{equation}
where $kv_b$ is the KV-cache footprint per token and $L_k$ is the mean context
length seen during decode for call $k$. Following the implementation, we
approximate the mid-decode context by
\begin{equation}
L_k = p_k + (m_k - 1)/2 .
\end{equation}

Summing over all decode tokens in one trajectory yields total decode traffic
\begin{equation}
\mathrm{BW}_{\mathrm{decode}} = \sum_k m_k \left(M_w / B + L_k \times kv_b\right) .
\end{equation}

The achievable decode batch size is constrained by the remaining HBM capacity
after reserving space for model weights and runtime scratch buffers:
\begin{equation}
B = \max\left(1, \left\lfloor \frac{\mathrm{HBM} - M_w - \mathrm{scratch}}{\bar{L} \times kv_b} \right\rfloor \right) ,
\end{equation}
where $\bar{L}$ is the decode-token-weighted mean context length over the
trajectory. This makes HBM capacity relevant even when the target metric is
bandwidth: a larger KV budget permits a larger $B$, which reduces the amortized
weight-read term $M_w / B$.

Decode time is then estimated by dividing total streamed bytes by peak HBM
bandwidth:
\begin{equation}
T_{\mathrm{decode}} = \mathrm{BW}_{\mathrm{decode}} / 4.8~\mathrm{TB/s} .
\end{equation}

Prefill is modeled separately as compute-bound:
\begin{equation}
T_{\mathrm{prefill}} = \mathrm{FLOPs}_{\mathrm{prefill}} / 989.0~\mathrm{TFLOPS} .
\end{equation}

The combined roofline estimate reported in the table is
\begin{equation}
T_{\mathrm{total}} = T_{\mathrm{decode}} + T_{\mathrm{prefill}} .
\end{equation}

\paragraph{3. Why context management changes HBM cost.}
Under this model, context management affects HBM cost through two coupled
mechanisms. First, shortening the live prompt directly reduces KV-cache traffic,
since the term $L_k \times kv_b$ is paid for every generated token. Second,
shorter contexts allow more requests to fit in the fixed 141.0 GB HBM budget,
which increases $B$ and further reduces the amortized weight-streaming term
$M_w / B$. As a result, methods that keep less verbatim context active during
multi-turn interaction can improve both \textbf{total HBM traffic} and
\textbf{roofline decode time} on the same hardware.

Because the same Qwen3-8B model and the same H200 device are used for all methods,
the measured differences in Table~\ref{tab:ablation_HBM_recall_needle} should be
interpreted as differences in \textbf{memory pressure induced by the
context-management policy}, rather than as differences in underlying model
architecture or hardware allocation.

\paragraph{4. Empirical trend on a single H200.}

Table~\ref{tab:ablation_HBM_recall_needle} shows that \textbf{ARC is consistently
the most HBM-efficient method} across all recall-budget settings. Relative to
the \textsc{Sliding\_Window} baseline, ARC reduces total decode traffic from
\textbf{30.31 TB to 9.85 TB} (\textbf{66.0\%} savings) at \texttt{4k / 2k},
from \textbf{32.82 TB to 10.48 TB} (\textbf{66.4\%} savings) at
\texttt{8k / 4k}, from \textbf{27.70 TB to 4.87 TB} (\textbf{80.3\%} savings)
at \texttt{16k / 8k}, and from \textbf{31.85 TB to 4.93 TB}
(\textbf{82.6\%} savings) at \texttt{32k / 16k}.

These bandwidth reductions translate into substantial roofline speedups:
\textbf{2.94$\times$}, \textbf{2.97$\times$}, \textbf{5.08$\times$}, and
\textbf{5.75$\times$} over \textsc{Sliding\_window}, respectively. ARC also
attains the highest success count in every setting ($925$, $932$, $991$, and
$994$), indicating that its lower HBM cost is not obtained by sacrificing task
effectiveness.

The remaining baselines are notably weaker on this hardware. \textsc{RAG\_Memory} is the closest competitor in absolute bandwidth, but its savings are modest and inconsistent ($6.7\%$, $8.4\%$, $-1.4\%$, and $13.3\%$).
\textsc{LLM\_Summary} and \textsc{Structured\_State} are above
\textsc{Sliding\_Window} in total HBM traffic for all four budget settings.
Overall, the single-H200 analysis shows that ARC's main systems advantage is
not only better task success, but a much smaller amount of HBM traffic per
solved trajectory, with the advantage widening as the recall budget becomes
less restrictive.

\begin{table*}[!htbp]
\centering
\small
\setlength{\tabcolsep}{6pt}
\begin{tabular}{l|l|ccrrrrr}
\toprule
\rowcolor{gray!12}
\textbf{Method} & \textbf{RECALL} & $\mathbf{n}_{\textbf{succ}}$ $\uparrow$ & $\mathbf{t}_{\textbf{tot}}$ $\downarrow$ & $\mathbf{t}_{\textbf{dec}}$ $\downarrow$ & $\mathbf{t}_{\textbf{pre}}$ $\downarrow$ & \textbf{BW} $\downarrow$ & \textbf{Savings (\%)} $\uparrow$ & \textbf{Speedup} $\uparrow$ \\
\rowcolor{gray!12}
& (\texttt{BUDGET\_CHARS}/\texttt{MAX\_CHARS}) & & (s) & (s) & (s) & (TB) & & ($\times$) \\
\midrule
\midrule
Sliding Window   & \multirow{5}{*}{4k / 2k}  & 570 & 6.970 & 6.316 & 0.655 & 30.3146 &   0.0 & 1.00 \\
LLM Summary      &                           & 762 & 8.619 & 8.112 & 0.507 & 38.9376 & -23.6 & 0.81 \\
Structured State &                           & 730 & 9.336 & 8.897 & 0.438 & 42.7060 & -33.9 & 0.75 \\
RAG Memory       &                           & 800 & 6.503 & 6.067 & 0.436 & 29.1240 &   6.7 & 1.07 \\
\textbf{ARC}     &                           & \textbf{925} & \textbf{2.367} & \textbf{2.051} & \textbf{0.316} & \textbf{9.8453} & \textbf{66.0} & \textbf{2.94} \\
\midrule
Sliding Window   & \multirow{5}{*}{8k / 4k}  & 534 & 7.512 & 6.837 & 0.675 & 32.8164 &   0.0 & 1.00 \\
LLM Summary      &                           & 784 & 8.380 & 7.879 & 0.501 & 37.8189 & -11.6 & 0.90 \\
Structured State &                           & 688 & 10.132 & 9.675 & 0.457 & 46.4407 & -34.9 & 0.74 \\
RAG Memory       &                           & 775 & 6.878 & 6.424 & 0.454 & 30.8357 &   8.4 & 1.09 \\
\textbf{ARC}     &                           & \textbf{932} & \textbf{2.527} & \textbf{2.184} & \textbf{0.343} & \textbf{10.4825} & \textbf{66.4} & \textbf{2.97} \\
\midrule
Sliding Window   & \multirow{5}{*}{16k / 8k} & 571 & 6.378 & 5.771 & 0.607 & 27.7007 &   0.0 & 1.00 \\
LLM Summary      &                           & 733 & 9.494 & 8.952 & 0.543 & 42.9673 & -48.9 & 0.67 \\
Structured State &                           & 688 & 10.589 & 10.115 & 0.474 & 48.5509 & -66.0 & 0.60 \\
RAG Memory       &                           & 788 & 6.467 & 6.025 & 0.442 & 28.9208 &  -1.4 & 0.99 \\
\textbf{ARC}     &                           & \textbf{991} & \textbf{1.256} & \textbf{1.014} & \textbf{0.241} & \textbf{4.8694} & \textbf{80.3} & \textbf{5.08} \\
\midrule
Sliding Window   & \multirow{5}{*}{32k / 16k}& 543 & 7.294 & 6.636 & 0.658 & 31.8519 &   0.0 & 1.00 \\
LLM Summary      &                           & 759 & 8.925 & 8.414 & 0.511 & 40.3860 & -22.4 & 0.82 \\
Structured State &                           & 699 & 9.910 & 9.456 & 0.454 & 45.3879 & -35.9 & 0.74 \\
RAG Memory       &                           & 793 & 6.325 & 5.890 & 0.435 & 28.2742 &  13.3 & 1.15 \\
\textbf{ARC}     &                           & \textbf{994} & \textbf{1.269} & \textbf{1.028} & \textbf{0.242} & \textbf{4.9322} & \textbf{82.6} & \textbf{5.75} \\
\bottomrule
\end{tabular}
\caption{HBM evaluation per successful trajectory for different methods under varying recall budgets. The baseline for savings and speedup is \textsc{Sliding\_window} within each recall budget.}
\label{tab:ablation_HBM_recall_needle}
\end{table*}

\textbf{The forced-compaction ablation} exhibits the same qualitative trend as the
recall-budget study: \textbf{ARC remains the most HBM-efficient method under every
compaction schedule}. In Table~\ref{tab:ablation_HBM_force_needle}, under the default \texttt{5,10,15} schedule, ARC reduces
per-success decode traffic from \textbf{27.70~TB} to \textbf{4.87~TB}
(\textbf{80.3\%} reduction) and decreases roofline time from \textbf{6.378~s}
to \textbf{1.256~s} (\textbf{5.08$\times$} speedup) relative to
\textsc{Sliding\_window}, while also achieving the highest number of successful
tasks (\textbf{991} versus \textbf{571}). The same pattern holds when
compaction is forced only at \texttt{5} or only at \texttt{15}: ARC again
achieves the lowest bandwidth consumption (\textbf{5.12~TB} and
\textbf{6.82~TB} per successful task, respectively), the lowest per-success
execution time (\textbf{1.314~s} and \textbf{1.709~s}), and the highest success
counts (\textbf{998} and \textbf{997}).

The strongest contrast appears under the most aggressive schedule,
\texttt{3,6,9,12,15}. Here, the \textsc{Sliding\_window} baseline succeeds on
only \textbf{189} tasks, whereas ARC succeeds on \textbf{989}. Because the
reported metrics are normalized \emph{per successful task}, this substantial
difference in the denominator makes the non-ARC baselines appear particularly
expensive on cost-per-success metrics. Even with this caveat, the overall
trend remains unchanged: ARC still achieves the smallest bandwidth footprint
(\textbf{4.96~TB per success}) and the lowest latency
(\textbf{1.290~s per success}), while \textsc{LLM\_summary},
\textsc{Structured\_state}, and \textsc{RAG\_memory} all remain substantially
more expensive. Overall, the forced-compaction results demonstrate that ARC's
HBM efficiency is robust not only to the recall budget, but also to the
frequency with which compaction is triggered.

\begin{table*}[!htbp]
\centering
\small
\setlength{\tabcolsep}{5pt}
\begin{tabular}{l|l|crrrrrrr}
\toprule
\rowcolor{gray!12}
\textbf{Method} & \textbf{FORCE} & $\mathbf{n}_{\textbf{succ}}$ $\uparrow$ & $\mathbf{t}_{\textbf{tot}}$ $\downarrow$ & $\mathbf{t}_{\textbf{dec}}$ $\downarrow$ & $\mathbf{t}_{\textbf{pre}}$ $\downarrow$ & \textbf{BW} $\downarrow$ & \textbf{Savings (\%)} $\uparrow$ & \textbf{Speedup} $\uparrow$ \\
\rowcolor{gray!12}
& & & (s) & (s) & (s) & (TB) & & ($\times$) \\
\midrule
\midrule
Sliding window   & \multirow{5}{*}{5, 10, 15} & 571 & 6.378 & 5.771 & 0.607 & 27.7007 & 0.0 & 1.00 \\
LLM summary      &                            & 733 & 9.494 & 8.952 & 0.543 & 42.9673 & -48.9 & 0.67 \\
Structured state &                            & 688 & 10.589 & 10.115 & 0.474 & 48.5509 & -66.0 & 0.60 \\
RAG memory       &                            & 788 & 6.467 & 6.025 & 0.442 & 28.9208 & -1.4 & 0.99 \\
\textbf{ARC}     &                            & \textbf{991} & \textbf{1.256} & \textbf{1.014} & \textbf{0.241} & \textbf{4.8694} & \textbf{80.3} & \textbf{5.08} \\
\cmidrule(l){1-9}
Sliding window   & \multirow{5}{*}{5}          & 666 & 5.174 & 4.613 & 0.561 & 22.1434 & 0.0 & 1.00 \\
LLM summary      &                            & 763 & 8.617 & 7.996 & 0.621 & 38.3788 & -66.5 & 0.60 \\
Structured state &                            & 769 & 10.431 & 9.928 & 0.504 & 47.6524 & -101.6 & 0.50 \\
RAG memory       &                            & 799 & 7.397 & 6.926 & 0.472 & 33.2431 & -43.0 & 0.70 \\
\textbf{ARC}     &                            & \textbf{998} & \textbf{1.314} & \textbf{1.066} & \textbf{0.248} & \textbf{5.1184} & \textbf{74.6} & \textbf{3.94} \\
\cmidrule(l){1-9}
Sliding window   & \multirow{5}{*}{15}         & 673 & 5.701 & 4.922 & 0.778 & 23.6278 & 0.0 & 1.00 \\
LLM summary      &                            & 765 & 9.470 & 8.667 & 0.803 & 41.6012 & -66.1 & 0.60 \\
Structured state &                            & 798 & 8.948 & 8.185 & 0.763 & 39.2882 & -57.0 & 0.64 \\
RAG memory       &                            & 790 & 6.856 & 6.327 & 0.529 & 30.3697 & -20.3 & 0.83 \\
\textbf{ARC}     &                            & \textbf{997} & \textbf{1.709} & \textbf{1.421} & \textbf{0.288} & \textbf{6.8212} & \textbf{70.0} & \textbf{3.34} \\
\cmidrule(l){1-9}
Sliding window   & \multirow{5}{*}{\parbox{2cm}{3, 6, 9,\\12, 15}} & 189 & 2.200 & 2.115 & 0.085 & 10.1509 & 0.0 & 1.00 \\
LLM summary      &                            & 876 & 4.944 & 4.508 & 0.436 & 21.6371 & -124.7 & 0.45 \\
Structured state &                            & 652 & 10.782 & 10.283 & 0.499 & 49.3566 & -390.1 & 0.20 \\
RAG memory       &                            & 799 & 9.090 & 8.688 & 0.402 & 41.7001 & -313.2 & 0.24 \\
\textbf{ARC}     &                            & \textbf{989} & \textbf{1.290} & \textbf{1.034} & \textbf{0.256} & \textbf{4.9618} & \textbf{41.4} & \textbf{1.71} \\
\bottomrule
\end{tabular}
\caption{HBM evaluation per successful trajectory for different methods on Qwen3-8B (16k context window, 4096 max-tokens) under varying FORCE configurations. The baseline for savings and speedup computations is the \textsc{Sliding\_window} method within each respective configuration block.}
\label{tab:ablation_HBM_force_needle}
\end{table*}

\subsubsection*{HBM for the Needle-in-a-Haystack Ablation: Fixture size ($\mathit{noise\_lines\_per\_side}$)}

The fixture-size ablation (Table~\ref{tab:ablation_HBM_noise_needle}) shows that \textbf{ARC remains the most HBM-efficient
method as the distractor observation becomes larger}. Across all four settings
(\texttt{5}, \texttt{10}, \texttt{25}, and \texttt{30} noise lines per side), ARC achieves the lowest total
time, the lowest decode time, and the smallest HBM traffic, while also
maintaining the highest success count. Relative to \textsc{Sliding\_window}, ARC reduces
bandwidth from \textbf{30.40~TB} to \textbf{3.07~TB} (\textbf{88.8\%} savings) at fixture size
\texttt{5}, from \textbf{25.93~TB} to \textbf{3.45~TB} (\textbf{84.9\%} savings) at \texttt{10}, from
\textbf{27.70~TB} to \textbf{4.87~TB} (\textbf{80.3\%} savings) at \texttt{25}, and from
\textbf{26.71~TB} to \textbf{5.09~TB} (\textbf{78.5\%} savings) at \texttt{30}.

These HBM savings translate into large roofline speedups throughout the sweep:
ARC improves per-success total time from \textbf{6.845~s} to \textbf{0.769~s}
(\textbf{8.90$\times$}) at fixture size \texttt{5}, from \textbf{5.871~s} to \textbf{0.885~s} (\textbf{6.63$\times$})
at \texttt{10}, from \textbf{6.378~s} to \textbf{1.256~s} (\textbf{5.08$\times$}) at \texttt{25}, and from
\textbf{6.221~s} to \textbf{1.339~s} (\textbf{4.65$\times$}) at \texttt{30}. The absolute ARC bandwidth and
latency do increase slightly as the fixture grows, which is expected because
larger distractor observations create longer live contexts and therefore more
KV-cache traffic. However, this degradation is much smaller for ARC than for
the baselines, indicating that citation-based compaction continues to suppress
HBM growth even when the retrieved observation itself becomes more expensive.

Among the non-ARC methods, \textsc{RAG\_memory} is again the closest competitor at the
smallest fixture sizes, with \textbf{17.2\%} savings at \texttt{5} and \textbf{34.4\%} at \texttt{10},
but it loses this advantage as the fixture grows and becomes worse than \textsc{Sliding\_window} at \texttt{25} and \texttt{30}. \textsc{LLM\_summary} and \textsc{Structured\_state} are
consistently above the \textsc{Sliding\_window} baseline in both bandwidth and total
time. Overall, the fixture-size sweep shows that ARC's HBM advantage is robust
not only to recall budget and compaction schedule, but also to the size of the
needle-containing observation that must be managed.

\begin{table*}[!htbp]
\centering
\small
\setlength{\tabcolsep}{5pt}
\begin{tabular}{l|l|crrrrrrr}
\toprule
\rowcolor{gray!12}
\textbf{Method} & \textbf{Noise Lines} & $\mathbf{n}_{\textbf{succ}}$ $\uparrow$ & $\mathbf{t}_{\textbf{tot}}$ $\downarrow$ & $\mathbf{t}_{\textbf{dec}}$ $\downarrow$ & $\mathbf{t}_{\textbf{pre}}$ $\downarrow$ & \textbf{BW} $\downarrow$ & \textbf{Savings (\%)} $\uparrow$ & \textbf{Speedup} $\uparrow$ \\
\rowcolor{gray!12}
& \textbf{(per side)} & & (s) & (s) & (s) & (TB) & & ($\times$) \\
\midrule
\midrule
Sliding window   & \multirow{5}{*}{5}  & 522 & 6.845 & 6.333 & 0.512 & 30.3991 & 0.0 & 1.00 \\
LLM summary      &                     & 738 & 8.824 & 8.438 & 0.386 & 40.5024 & -28.9 & 0.78 \\
Structured state &                     & 676 & 10.164 & 9.761 & 0.403 & 46.8541 & -48.5 & 0.67 \\
RAG memory       &                     & 797 & 5.666 & 5.361 & 0.305 & 25.7347 & 17.2 & 1.21 \\
\textbf{ARC}     &                     & \textbf{999} & \textbf{0.769} & \textbf{0.639} & \textbf{0.130} & \textbf{3.0658} & \textbf{88.8} & \textbf{8.90} \\
\cmidrule(l){1-9}
Sliding window   & \multirow{5}{*}{10} & 603 & 5.871 & 5.402 & 0.469 & 25.9292 & 0.0 & 1.00 \\
LLM summary      &                     & 790 & 7.894 & 7.516 & 0.378 & 36.0761 & -34.4 & 0.74 \\
Structured state &                     & 743 & 8.864 & 8.494 & 0.370 & 40.7730 & -51.0 & 0.66 \\
RAG memory       &                     & 872 & 3.852 & 3.599 & 0.253 & 17.2744 & 34.4 & 1.52 \\
\textbf{ARC}     &                     & \textbf{1000} & \textbf{0.885} & \textbf{0.718} & \textbf{0.167} & \textbf{3.4457} & \textbf{84.9} & \textbf{6.63} \\
\cmidrule(l){1-9}
Sliding window   & \multirow{5}{*}{25} & 571 & 6.378 & 5.771 & 0.607 & 27.7007 & 0.0 & 1.00 \\
LLM summary      &                     & 733 & 9.494 & 8.952 & 0.543 & 42.9673 & -48.9 & 0.67 \\
Structured state &                     & 688 & 10.589 & 10.115 & 0.474 & 48.5509 & -66.0 & 0.60 \\
RAG memory       &                     & 788 & 6.467 & 6.025 & 0.442 & 28.9208 & -1.4 & 0.99 \\
\textbf{ARC}     &                     & \textbf{991} & \textbf{1.256} & \textbf{1.014} & \textbf{0.241} & \textbf{4.8694} & \textbf{80.3} & \textbf{5.08} \\
\cmidrule(l){1-9}
Sliding window   & \multirow{5}{*}{30} & 631 & 6.221 & 5.566 & 0.655 & 26.7149 & 0.0 & 1.00 \\
LLM summary      &                     & 750 & 10.109 & 9.523 & 0.586 & 45.7103 & -62.5 & 0.62 \\
Structured state &                     & 731 & 10.083 & 9.615 & 0.467 & 46.1528 & -62.1 & 0.62 \\
RAG memory       &                     & 755 & 7.805 & 7.292 & 0.513 & 35.0015 & -25.5 & 0.80 \\
\textbf{ARC}     &                     & \textbf{992} & \textbf{1.339} & \textbf{1.060} & \textbf{0.279} & \textbf{5.0904} & \textbf{78.5} & \textbf{4.65} \\
\bottomrule
\end{tabular}
\caption{Evaluation results per successful trajectory for various memory mechanisms on Qwen3-8B (16k context window, 4096 max-tokens) under different structural noise conditions (\texttt{NOISE\_LINES\_PER\_SIDE}).}
\label{tab:ablation_HBM_noise_needle}
\end{table*}

%% file: cross_platform.tex
\section{Cross-Platform Validation: AMD MI300X}
\label{sec:cross-platform}

To verify that ARC's advantage is not an artifact of a single vendor's
hardware or numerical stack, we replicate the Needle-in-a-Haystack
evaluation from the main paper on AMD Instinct MI300X GPUs (192\,GB HBM3,
running ROCm) using the identical experimental protocol: Qwen3-8B at a
16k context window and Qwen3-32B at a 32k context window, both under a
4,096-token completion cap, with the same three seeds (42, 82, 122).

\subsection{Accuracy}

Table~\ref{tab:supp-amd-accuracy} reports mean accuracy across the three
seeds on AMD MI300X, directly comparable to Table~2 in the main paper
(NVIDIA H200).

\begin{table}[t]
\centering
\setlength{\tabcolsep}{4pt}
\begin{tabular}{llrrr}
\toprule
\rowcolor{gray!12}
Method & Pass/Att. & Mean Acc$\uparrow$ & SD Acc \\
\midrule
\midrule
\rowcolor{blue!20}
\multicolumn{4}{l}{\textit{Qwen3-8B}} \\
Full\_context      & 1592/3000 & 53.07\% & 1.45\% \\
Sliding\_window                                & 2055/3000 & 68.50\% & 0.35\% \\
LLM\_summary                                   & 2507/3000 & 83.57\% & 1.99\% \\
Structured\_state                              & 2355/3000 & 78.50\% & 0.89\% \\
RAG\_memory                                    & 2431/3000 & 81.03\% & 1.46\% \\
\textbf{ARC}                                   & \textbf{2987/3000} & \textbf{99.57\%} & 0.35\% \\
\midrule
\rowcolor{blue!20}
\multicolumn{4}{l}{\textit{Qwen3-32B}} \\
Full\_context      &  118/3000 &  3.93\% & 1.08\% \\
Sliding\_window                                & 2863/3000 & 95.43\% & 0.35\% \\
LLM\_summary                                   & 2815/3000 & 93.83\% & 0.68\% \\
Structured\_state                              & 2882/3000 & 96.07\% & 0.59\% \\
RAG\_memory                                    & 2925/3000 & 97.50\% & 0.17\% \\
\textbf{ARC}                                   & \textbf{2993/3000} & \textbf{99.77\%} & 0.23\% \\
\bottomrule
\end{tabular}
\caption{Needle-in-a-Haystack accuracy on AMD MI300X, 3 seeds (42, 82,
122) of 1000 tasks each (3000 attempts/method). Directly comparable to
Table~2 (NVIDIA H200) in the main paper.}
\label{tab:supp-amd-accuracy}
\end{table}

ARC's accuracy is essentially unchanged across vendors: 99.00\%
(H200) versus 99.57\% (MI300X) on Qwen3-8B, and 99.80\% versus 99.77\%
on Qwen3-32B, in both cases within the seed-to-seed noise already
present on a single platform. This indicates that ARC's recovery
guarantee holds independent of the underlying accelerator.

The baseline methods show a different pattern. At Qwen3-8B, every
baseline scores substantially higher on MI300X than on H200
(e.g., Full\_context rises from 36.80\% to 53.07\%, LLM\_summary from
71.77\% to 83.57\%). At Qwen3-32B, this gap largely disappears: every
method, baselines included, lands within roughly one point of its H200
counterpart. Because ARC's own accuracy is stable across both
configurations, this discrepancy affects only the comparison among
lossy baselines and does not bear on ARC's central claim.

\subsection{Hardware-Cost Model}

We additionally replay each method's recorded trajectory through the
same roofline hardware-cost model used in the main paper
(Eq.~1--2), substituting AMD MI300X specifications (192\,GB HBM3,
5.3\,TB/s peak bandwidth, 1300\,TFLOPS BF16 dense compute) for the
NVIDIA H200 figures. Following the same
per-success normalization as the main paper, and at the request of a
collaborator to compare against a compaction baseline rather than the
non-compacting control, we report savings relative to
\emph{Sliding\_window} rather than \emph{Full\_context}.

\begin{table}[t]
\centering
\setlength{\tabcolsep}{4pt}
\begin{tabular}{lrrrr}
\toprule
\rowcolor{gray!12}
Method & t\_tot.(s)$\downarrow$ & bw(TB)$\downarrow$ & Sav.(\%)$\uparrow$ \\
\midrule
\midrule
\rowcolor{blue!20}
\multicolumn{4}{l}{\textit{Qwen3-8B}} \\
Sliding\_window   &  4.942 & 23.92 &  0.0 \\
LLM\_summary      &  6.293 & 31.45 & -27.3 \\
Structured\_state &  7.853 & 39.99 & -58.9 \\
RAG\_memory       &  5.134 & 25.56 & -3.9 \\
\textbf{ARC}      &  \textbf{1.009} &  \textbf{4.38} & \textbf{79.6} \\
\midrule
\rowcolor{blue!20}
\multicolumn{4}{l}{\textit{Qwen3-32B}} \\
Sliding\_window   &  6.010 & 26.79 &  0.0 \\
LLM\_summary      &  3.859 & 16.83 & 35.8 \\
Structured\_state &  2.079 &  8.18 & 65.4 \\
RAG\_memory       &  3.462 & 15.29 & 42.4 \\
\textbf{ARC}      &  \textbf{1.572} &  \textbf{4.32} & \textbf{73.8} \\
\bottomrule
\end{tabular}
\caption{High-bandwidth-memory comparison on AMD MI300X, Needle-in-a-Haystack,
same 4,096-token completion cap as the main paper (seed 42). Savings reported
relative to Sliding\_window, per the request of a project collaborator to
compare against a compaction-based baseline rather than the non-compacting
control.}
\label{tab:supp-amd-hbm}
\end{table}

ARC is the fastest and lowest-bandwidth method at both scales. On
Qwen3-8B, ARC completes each success in 1.01\,s at 4.38\,TB of HBM
traffic, a 79.6\% time saving relative to Sliding\_window; every other
lossy baseline is in fact \emph{slower} than Sliding\_window at this
scale (Sav.\ is negative for LLM\_summary, Structured\_state, and
RAG\_memory), so ARC is the only method that both compacts and
is faster than the naive windowing baseline it is compared against. On
Qwen3-32B, ARC again leads on both time (1.57\,s) and bandwidth
(4.32\,TB), though Structured\_state narrows the gap on this larger
model (65.4\% savings versus ARC's 73.8\%).

\subsection{Summary}

These results indicate that ARC's core contribution, exact,
addressable recovery of compacted content, transfers across accelerator
vendors without modification. The magnitude of ARC's efficiency
advantage over lossy baselines is consistent with the main paper's
NVIDIA findings, supporting the claim that ARC's benefits stem from its
recall mechanism rather than from vendor-specific kernel behavior.